%% file: main.tex
\documentclass{article}

\usepackage{microtype}
\usepackage{graphicx}
\usepackage{subcaption}
\usepackage{booktabs} 
\usepackage{enumitem}

\usepackage{hyperref}

\usepackage[preprint]{icml2026}

\usepackage{amsmath,amssymb,amsthm}
\usepackage{algorithm}
\usepackage{algorithmic}
\usepackage{multirow}
\usepackage{xcolor}
\usepackage{float}
\usepackage{tikz}
\usetikzlibrary{shapes.geometric, arrows.meta, positioning, calc, fit, backgrounds, decorations.pathreplacing}
\usepackage{tcolorbox}
\usepackage{dsfont}
\newtheorem{theorem}{Theorem}

\newtheorem{corollary}[theorem]{Corollary}
\newtheorem{proposition}[theorem]{Proposition}

\newcommand{\E}{\mathbb{E}}

\newcommand{\Prob}{\mathbb{P}}
\newcommand{\bpi}{\boldsymbol{\pi}}
\newcommand{\bz}{\mathbf{z}}

\newcommand{\by}{\mathbf{y}}
\newcommand{\bw}{\mathbf{w}}
\newcommand{\bG}{\mathbf{G}}

\newcommand{\bDelta}{\boldsymbol{\Delta}}

\definecolor{softcolor}{RGB}{255,200,200}
\definecolor{hardcolor}{RGB}{200,230,255}
\definecolor{gumbelcolor}{RGB}{255,235,200}
\definecolor{detercolor}{RGB}{200,255,200}
\definecolor{highlightcolor}{RGB}{144,238,144}
\icmltitlerunning{Align Forward, Adapt Backward: Closing the Discretization Gap in Logic Gate Networks}

\begin{document}

\twocolumn[
  \icmltitle{Align Forward, Adapt Backward: \\
Closing the Discretization Gap in Logic Gate Networks}

\icmlsetsymbol{equal}{*}

  \begin{icmlauthorlist}
    \icmlauthor{Youngsung Kim}{yyy}
  \end{icmlauthorlist}

  \icmlaffiliation{yyy}{Department of Artificial Intelligence and Department of Electrical and Computer Engineering, Inha University, Incheon, Korea}
  \icmlcorrespondingauthor{Youngsung Kim}{y.kim@inha.ac.kr;yskim.ee@gmail.com}
  
\icmlkeywords{differentiable discrete selection, training-inference mismatch, forward alignment, straight-through estimator, logic gate networks}

\vskip 0.3in
]
\printAffiliationsAndNotice{}  

\begin{abstract}
In neural network models, soft mixtures of fixed candidate components (e.g., logic gates and sub-networks) are often used during training 
for stable optimization, while hard selection is typically used at inference. This raises questions about training-inference mismatch. We analyze this gap by separating forward-pass computation (hard selection vs.\ soft mixture) from stochasticity (with vs.\ without Gumbel noise). Using logic gate networks as a testbed, we observe distinct behaviors across four methods: Hard-ST achieves zero selection gap by construction; Gumbel-ST achieves near-zero gap when training succeeds but suffers accuracy collapse at low temperatures; Soft-Mix achieves small gap only at low temperature via weight concentration; and Soft-Gumbel exhibits large gaps despite Gumbel noise, confirming that noise alone does not reduce the gap. We propose CAGE (Confidence-Adaptive Gradient Estimation) to maintain gradient flow while preserving forward alignment. On logic gate networks, Hard-ST with CAGE achieves over 98\% accuracy on MNIST and over 58\% on CIFAR-10, both with zero selection gap across all temperatures, while Gumbel-ST without CAGE suffers a 47-point accuracy collapse.

\end{abstract}

\input{section1_introduction}

\input{section2_background}
\input{section3_revisiting}      
\input{section4_theory}          
\input{section5_cage}            
\input{section6_experiments}     
\input{section9_conclusion}      

\input{impact_statement}

\bibliography{references}
\bibliographystyle{icml2026}

\newpage
\appendix
\onecolumn
\appendix
\input{appendix}

\input{appendix_proofs}
\input{appendix_experiments}
\input{appendix_input_distribution}

\end{document}

%% file: section1_introduction.tex

\section{Introduction}
\label{sec:introduction}

Many neural network architectures require selecting among discrete options 
during computation~\citep{mohamed2020monte, heuillet2024dnas, cai2025moe}. 
Differentiable logic gate networks~\citep{petersen2022deep, kim2023deep, 
yousefi2025mindthegap} learn Boolean circuits by choosing among gate 
operations (AND, OR, XOR, etc.) at each neuron. Similar discrete selection 
arises in neural architecture search~\citep{liu2019darts}, mixture-of-experts 
routing~\citep{jacobs1991adaptive, fedus2021switch}, and vector 
quantization~\citep{oord2017neural}. When hard selection is required at 
inference, a fundamental challenge emerges: methods using soft mixtures 
during training but hard selection at deployment suffer from 
\emph{training-inference mismatch}, while methods using hard selection 
throughout avoid this gap (Figure~\ref{fig:intro_overview}).

\input{fig_intro}

Recent work has achieved substantial gap reduction using the Straight-Through Gumbel-Softmax (Gumbel-ST) estimator~\citep{jang2017categorical, maddison2017concrete, xie2019snas}, which combines the Gumbel-Max trick~\citep{gumbel1954statistical} with straight-through gradient estimation~\citep{bengio2013estimating}: hard selection (argmax with Gumbel noise) in the forward pass, with gradients computed through the continuous Gumbel-Softmax relaxation in the backward pass~\citep{kim2023deep,yousefi2025mindthegap}. This success has been attributed to \emph{implicit Hessian regularization}—the hypothesis that Gumbel noise smooths the loss landscape, leading to solutions that generalize better to discrete inference~\citep{yousefi2025mindthegap}. We revisit this explanation and find it incomplete. Through controlled experiments, we show that methods with identical Gumbel noise exhibit markedly different gaps depending on whether their forward pass uses soft mixtures or hard selection. This suggests a simpler mechanism: gap reduction occurs because Gumbel-ST performs \emph{hard selection during training}, aligning training computation with deployment.

We call this principle \emph{forward alignment}: the forward pass during training should match the forward pass at deployment, eliminating training-inference mismatch. While conceptually straightforward, hard selection has been avoided due to concerns about biased gradient estimates from the straight-through approximation~\citep{bengio2013estimating} and the high variance of score-function alternatives like REINFORCE~\citep{williams1992simple}. Our central contribution is demonstrating that these gradient estimation challenges can be addressed \emph{entirely in the backward pass}—through temperature scheduling and adaptive gradient estimation—without compromising forward alignment. This separation is possible because backward-pass temperature does not affect the forward output for hard selection methods, a property we call \emph{temperature decoupling}.

\paragraph{Why does forward alignment matter?}
Our theoretical analysis reveals a fundamental asymmetry between mixture and hard selection methods. Mixture methods can minimize training loss by \emph{hedging}: distributing weight across multiple options to produce a blended output, without any single option achieving low loss individually. At deployment, argmax forces commitment to one option, exposing this lack of preparation. Hard selection methods cannot hedge—they must \emph{commit} to options that individually perform well, because only one option contributes to each forward pass. This distinction between hedging and commitment, formalized in Section~\ref{sec:theory}, explains why forward alignment eliminates the discretization gap.

\paragraph{Contributions.}

\begin{enumerate}
[nosep, leftmargin=*]
    \item \textbf{Theoretical framework.} We decompose the training-inference gap into \emph{selection gap} (method-dependent, reducible) and \emph{computation gap} (input-dependent, irreducible), and prove that hard selection methods achieve zero or near-zero selection gap regardless of backward temperature (Corollary~\ref{cor:gumbel_st_gap}, Proposition~\ref{prop:hard_st_gap}). We formalize the \emph{hedging capacity} of mixture methods and prove that hard selection eliminates it (Proposition~\ref{prop:competition}).
    
    \item \textbf{Empirical separation of causal factors.} Using a $2 \times 2$ factorial design crossing forward type (mixture vs.\ hard) with noise type (deterministic vs.\ Gumbel), we provide direct evidence that gap depends on forward structure, not gradient estimation. Critically, Soft-Gumbel and Gumbel-ST use identical noise distributions but exhibit dramatically different gaps, challenging the Hessian regularization hypothesis.
    
    \item \textbf{Discovery of a failure mode.} We identify that Gumbel-ST suffers catastrophic accuracy collapse at low temperatures (dropping from ${\sim}98\%$ to ${\sim}50\%$), while deterministic Hard-ST remains stable. This has practical implications for temperature scheduling.
    
    \item \textbf{CAGE algorithm.} We propose Confidence-Adaptive Gate Exploration, which exploits temperature decoupling to adapt backward temperature based on selection confidence, enabling exploration early and exploitation late without affecting alignment guarantees.
\end{enumerate}

%% file: fig_intro.tex

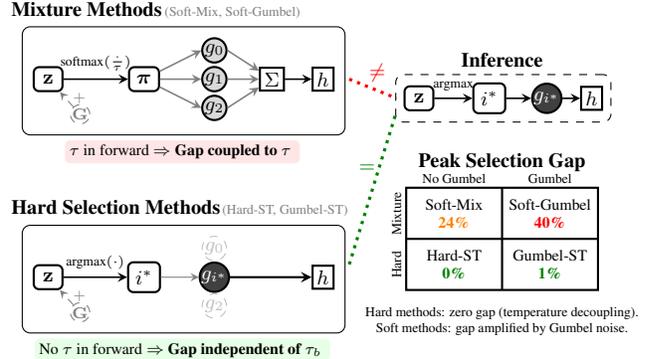
\begin{figure}[tb]
\centering
\begin{tikzpicture}[
    scale=0.85, transform shape,
    gate/.style={circle, draw, thick, minimum size=0.38cm, font=\footnotesize, inner sep=0pt},
    selgate/.style={gate, fill=black!75, text=white},
    deselgate/.style={circle, gray!50, draw, dashed, minimum size=0.38cm, font=\footnotesize, inner sep=0pt},
    mixgate/.style={gate, fill=gray!30},
    obs/.style={rectangle, draw, thick, minimum size=0.30cm, font=\footnotesize, inner sep=2pt},
    logit/.style={rectangle, draw, thick, rounded corners=2pt, minimum width=0.45cm, minimum height=0.24cm, font=\footnotesize},
    weight/.style={rectangle, draw, thick, rounded corners=2pt, minimum width=0.38cm, minimum height=0.22cm, font=\footnotesize},
    noise/.style={circle, draw, dashed, minimum size=0.24cm, font=\scriptsize, inner sep=0pt, gray},
    plate/.style={rectangle, draw, rounded corners=3pt, inner sep=4pt},
    arr/.style={->, >=stealth, semithick},
    selarr/.style={arr, thick},
    mixarr/.style={arr, gray},
    noisearr/.style={->, >=stealth, dashed, thin, gray}
]

\node[font=\small\bfseries, anchor=west] at (-4.9, 2.0) {Mixture Methods};
\node[font=\tiny, gray, anchor=west] at (-2.5, 1.95) {(Soft-Mix, Soft-Gumbel)};

\node[logit] (sz) at (-4.2, 0.9) {$\mathbf{z}$};
\node[noise] (sg) at (-3.7, 0.35) {$\mathbf{G}$};
\draw[noisearr] (sg) -- node[right, font=\tiny, gray, pos=0.6] {$+$} (sz);

\node[weight] (spi) at (-2.7, 0.9) {$\boldsymbol{\pi}$};
\draw[arr] (sz) -- node[above, font=\tiny] {softmax$(\frac{\cdot}{\tau})$} (spi);

\node[mixgate] (s0) at (-1.6, 1.35) {$g_0$};
\node[mixgate] (s1) at (-1.6, 0.9) {$g_1$};
\node[mixgate] (s2) at (-1.6, 0.45) {$g_2$};
\node[obs] (ssum) at (-0.7, 0.9) {$\Sigma$};
\node[obs] (so) at (0.1, 0.9) {$h$};

\draw[mixarr] (spi) -- (s0.west);
\draw[mixarr] (spi) -- (s1.west);
\draw[mixarr] (spi) -- (s2.west);
\draw[mixarr] (s0.east) -- (ssum);
\draw[mixarr] (s1.east) -- (ssum);
\draw[mixarr] (s2.east) -- (ssum);
\draw[arr] (ssum) -- (so);

\node[plate, fit=(sz)(sg)(spi)(s0)(s2)(ssum)(so)] {};

\node[font=\scriptsize, anchor=north, align=center, fill=red!10, rounded corners=2pt, inner sep=2pt] at (-2.1, -0.05) {
    $\tau$ in forward $\Rightarrow$ \textbf{Gap coupled to $\tau$}
};

\node[font=\small\bfseries, anchor=west] at (-4.9, -1.1) {Hard Selection Methods};
\node[font=\tiny, gray, anchor=west] at (-1.6, -1.15) {(Hard-ST, Gumbel-ST)};

\node[logit] (hz) at (-4.2, -2.2) {$\mathbf{z}$};
\node[noise] (hg) at (-3.7, -2.75) {$\mathbf{G}$};
\draw[noisearr] (hg) -- node[right, font=\tiny, gray, pos=0.6] {$+$} (hz);

\node[weight] (hsel) at (-2.7, -2.2) {$i^*$};
\draw[arr] (hz) -- node[above, font=\tiny] {argmax$(\cdot)$} (hsel);

\node[deselgate] (h0) at (-1.6, -1.75) {$g_0$};
\node[selgate] (h1) at (-1.6, -2.2) {$g_{i^*}$};
\node[deselgate] (h2) at (-1.6, -2.65) {$g_2$};
\node[obs] (ho) at (0.1, -2.2) {$h$};

\draw[arr, gray, thin] (hsel) -- (h1.west);
\draw[selarr] (h1.east) -- (ho);

\node[plate, fit=(hz)(hg)(hsel)(h0)(h2)(ho)] {};

\node[font=\scriptsize, anchor=north, align=center, fill=green!10, rounded corners=2pt, inner sep=2pt] at (-2.1, -3.15) {
    No $\tau$ in forward $\Rightarrow$ \textbf{Gap independent of $\tau_b$}
};


\node[font=\small\bfseries] at (2.9, 1.2) {Inference};

\node[logit] (iz) at (1.6, 0.6) {$\mathbf{z}$};
\node[weight] (isel) at (2.7, 0.6) {$i^*$};
\draw[arr] (iz) -- node[above, font=\tiny] {argmax} (isel);
\node[selgate] (igate) at (3.6, 0.6) {$g_{i^*}$};
\draw[arr] (isel) -- (igate);
\node[obs] (io) at (4.3, 0.6) {$h$};
\draw[arr] (igate) -- (io);

\node[plate, dashed, fit=(iz)(isel)(igate)(io), inner sep=3pt] {};

\node[font=\small\bfseries] at (2.9, -0.4) {Peak Selection Gap};

\draw[thick] (1.4, -0.8) rectangle (4.4, -2.4);
\draw[thick] (2.9, -0.8) -- (2.9, -2.4);  
\draw[thick] (1.4, -1.6) -- (4.4, -1.6);  

\node[font=\tiny] at (2.15, -0.65) {No Gumbel};
\node[font=\tiny] at (3.65, -0.65) {Gumbel};
\node[font=\tiny, rotate=90, anchor=south] at (1.45, -1.2) {Mixture};
\node[font=\tiny, rotate=90, anchor=south] at (1.45, -2.0) {Hard};

\node[font=\scriptsize, align=center] at (2.15, -1.2) {Soft-Mix\\{\color{orange}\textbf{24\%}}};
\node[font=\scriptsize, align=center] at (3.65, -1.2) {Soft-Gumbel\\{\color{red}\textbf{40\%}}};
\node[font=\scriptsize, align=center] at (2.15, -2.0) {Hard-ST\\{\color{green!50!black}\textbf{0\%}}};
\node[font=\scriptsize, align=center] at (3.65, -2.0) {Gumbel-ST\\{\color{green!50!black}\textbf{1\%}}};

\node[font=\tiny, align=center, anchor=north] at (2.9, -2.55) {
    Hard methods: zero gap (temperature decoupling).\\
    Soft methods: gap amplified by Gumbel noise.
};

\draw[dotted, very thick, red] (0.51, 0.9) -- (1.25, 0.6);
\node[font=\small, red] at (0.95, 1.0) {$\neq$};

\draw[dotted, very thick, green!50!black] (0.51, -2.0) -- (1.25, 0.4);
\node[font=\small, green!50!black] at (0.79, -0.5) {$=$};

\end{tikzpicture}

\caption{\textbf{Training-inference gap depends primarily on forward-pass structure.} \textbf{Left:} Mixture methods (top) compute weighted sums, enabling hedging; hard selection methods (bottom) commit to a single gate, matching inference. \textbf{Right-Top:} Inference uses argmax selection. \textbf{Right-Bottom:} Peak selection gap ($2 \times 2$ factorial, MNIST-Binary, $L$=6, $\tau$=2.0). Hard methods achieve near-zero gap; soft methods show 24--40\% gap, with Gumbel noise amplifying the gap.}
\label{fig:intro_overview}
\end{figure} 

%% file: section2_background.tex

\section{Background and Related Work}
\label{sec:background}

\subsection{Discrete Selection and the Training-Inference Gap}

Many neural network components must select among $K$ discrete choices, each producing output $g_i$. Examples include Boolean gates in logic gate networks~\citep{petersen2022deep}, candidate operations in neural architecture search~\citep{liu2019darts}, and expert networks in mixture of experts~\citep{zhou2022mixture, puigcerver2024soft, jacobs1991adaptive}. The standard approach parameterizes selection with logits $\bz \in \mathbb{R}^K$, converted to probabilities via softmax~\cite{maddison2017concrete, jang2017categorical, liu2019darts}:
\begin{equation}\label{eq:softmax}
\pi_i(\tau) = \frac{\exp(z_i/\tau)}{\sum_j \exp(z_j/\tau)},
\end{equation}
where temperature $\tau$ controls sharpness ($\tau \to 0$: one-hot; $\tau \to \infty$: uniform).

The challenge is that training and deployment handle these probabilities differently. During training, \emph{soft selection} computes a differentiable weighted mixture $h_{\text{soft}} = \sum_{i} \pi_i(\tau) g_i$, enabling gradient-based optimization. At deployment, \emph{hard selection} picks a single choice $h_{\text{hard}} = g_{i^*}$ where $i^* = \arg\max_i z_i$, as required for efficiency and interpretability~\citep{liu2019darts}. This structural difference between blending all choices and picking one creates the training-inference gap, even when individual $g_i$ are computed identically in both phases.

A middle ground is \emph{categorical sampling} via Gumbel-max~\citep{jang2017categorical,maddison2017concrete}:
\begin{equation}\label{eq:gumbel_sample}
h_{\text{sample}} = g_k, \quad k = \arg\max_i (z_i + G_i),
\end{equation}
where $G_i$ are i.i.d.\ standard Gumbel variables (CDF: $F(x) = e^{-e^{-x}}$). The Gumbel-max theorem ensures $\Prob(k = i) = \mathrm{softmax}(\bz)_i$. Crucially, this probability depends only on logits, not temperature, since argmax is scale-invariant~\citep{gumbel1954statistical}. This property is central to our analysis.

\subsection{Straight-Through Estimators}

Hard selection is non-differentiable, but the straight-through (ST) estimator~\citep{bengio2013estimating} enables gradient-based training by decoupling forward and backward computation: the forward pass uses discrete selection, while the backward pass computes gradients through a continuous surrogate.

\textbf{Gumbel-ST}~\citep{jang2017categorical, maddison2017concrete} applies this principle to categorical sampling. The forward pass samples a discrete choice via~\eqref{eq:gumbel_sample}, while the backward pass differentiates through $\tilde{\bpi}(\tau) = \mathrm{softmax}((\bz + \mathbf{G})/\tau)$. Here, temperature $\tau$ controls gradient smoothness without affecting forward selection. Section~\ref{sec:framework} formalizes this as \emph{temperature decoupling}.


\subsection{Challenges of Hard Selection}

If hard selection eliminates mismatch, why is it often avoided? The literature documents several concerns~\citep{bengio2013estimating,jang2017categorical,liu2019darts}. Sparse gradients mean only selected choices directly influence learning. Deterministic argmax may lock into suboptimal early choices. Discrete selections can flip abruptly, causing instability. Convergence is typically slower without smooth gradient landscapes~\citep{zela2020understanding}. These concerns have led practitioners toward soft methods despite their mismatch problem, or toward domain-specific workarounds. Our approach resolves this tension by showing that backward-pass adaptation can address these optimization challenges while forward alignment preserves zero mismatch.

\subsection{Related Work}

The training-inference gap manifests across domains, with each community developing its own solutions. In neural architecture search, DARTS~\citep{liu2019darts} exemplifies the soft relaxation paradigm and the resulting discretization gap. Subsequent methods such as SNAS~\citep{xie2019snas} and GDAS~\citep{dong2019gdas} introduce stochastic or near-hard operation selection to mitigate this mismatch. In large-scale mixture-of-experts models, hard routing has been predominantly adopted~\citep{fedus2021switch}, effectively avoiding training–inference inconsistency. Similarly, VQ-VAE~\citep{oord2017neural} employs hard codebook selection together with auxiliary commitment losses to stabilize learning. By contrast, differentiable logic gate networks~\citep{petersen2022deep} typically rely on soft gate mixtures, which can introduce a soft-to-hard discrepancy at inference time. \citet{kim2023deep} introduced the Gumbel-ST framework for logic gate networks, relieving this discrepancy, and \citet{yousefi2025mindthegap} attribute its improvement to Hessian regularization acting as a smooth optimization objective.

Our analysis offers a unifying perspective: across these domains, 
methods succeed when forward computation aligns with deployment, 
regardless of the specific noise injection or regularization employed. Temperature scheduling has been explored via sparsifying small 
temperatures~\citep{zhang2023small} and decoupled forward-backward 
temperatures~\citep{shah2024decoupled}; our framework explains \emph{why} such approaches succeed specifically for hard forward methods.

%% file: section3_revisiting.tex

\section{Forward Alignment Framework}
\label{sec:framework}

Section~\ref{sec:background} introduced soft and hard selection. We now characterize where the training-inference gap arises and which design choices eliminate it.

\subsection{Gap Decomposition}
\label{sec:gap_decomposition}

The gap between training and deployment has two potential sources: (1) the selection mechanism (mixture vs.\ argmax), and (2) the gate evaluation (soft vs.\ hard; see Appendix~\ref{app:gate_definitions} for definitions). To isolate these, we define three accuracy measures on the test set.

Let $i^* = \arg\max_i z_i$ denote the selected gate. $A^{M}$ uses method $M$'s forward computation, matching the training setting. $A^{\text{soft}}$ uses argmax selection with soft gate evaluation $g_{i^*}^{\text{soft}}$. $A^{\text{hard}}$ uses argmax selection with hard gate evaluation $g_{i^*}^{\text{hard}}$, matching the deployment setting.

The total gap decomposes as:
\begin{equation}\label{eq:gap_decomp}
\underbrace{A^{M} - A^{\text{hard}}}_{\text{Total Gap}} = 
\underbrace{A^{M} - A^{\text{soft}}}_{\text{Selection Gap}} + 
\underbrace{A^{\text{soft}} - A^{\text{hard}}}_{\text{Computation Gap}}
\end{equation}

\textbf{Selection gap} isolates the effect of selection mechanism. For mixture methods, training uses $\sum_i \pi_i g_i^{\text{soft}}$ while deployment uses single gate $g_{i^*}^{\text{soft}}$; these differ. For hard selection methods, training already uses $g_{i^*}^{\text{soft}}$, so $A^M = A^{\text{soft}}$ and selection gap is zero.

\textbf{Computation gap} isolates the effect of gate evaluation, comparing soft versus hard gates for the same selected gate. This depends only on input distribution: for binary inputs, soft and hard gates are identical (zero gap), while continuous inputs produce an unavoidable gap (Appendix~\ref{app:input_distribution}). Importantly, computation gap is identical across all methods.

The decomposition reveals an asymmetry: selection gap is method-dependent and reducible, whereas computation gap is input-dependent, irreducible, yet method-independent. (In the remainder of the paper, we write $g_i$ without superscript; the soft/hard distinction is relevant only for computation gap.)

\subsection{Forward Alignment and Temperature Decoupling}
\label{sec:forward_alignment}

\paragraph{Forward alignment.}
Figure~\ref{fig:intro_overview} organizes methods along two axes: selection type (mixture vs.\ hard) and noise type (deterministic vs.\ Gumbel). The key insight is that selection type (row), not noise type (column), determines selection gap. If training uses argmax selection (hard selection methods), it matches deployment exactly, regardless of whether Gumbel noise is present. We call such methods \emph{forward-aligned}.

This yields predictions that contradict prior work attributing gap reduction to Gumbel noise~\citep{yousefi2025mindthegap}: (1) Hard-ST should have smaller gap than Soft-Gumbel, despite lacking Gumbel noise; (2) adding Gumbel noise to mixture methods (Soft-Mix $\to$ Soft-Gumbel) should not reduce selection gap. Section~\ref{sec:experiments} tests these predictions.

\paragraph{Temperature decoupling.}
Forward alignment provides a practical benefit beyond gap reduction. In mixture methods, temperature $\tau$ controls the forward output through $\pi_i(\tau) = \mathrm{softmax}(\bz/\tau)_i$, so changing $\tau$ alters the mixture weights and thereby the selection gap. Temperature and gap are coupled.

In hard selection methods, the forward pass uses argmax, which is scale-invariant: $\arg\max_i z_i = \arg\max_i (z_i/\tau)$ for any $\tau > 0$. Temperature enters only the backward pass through the ST surrogate gradient. We call this the \emph{backward temperature} $\tau_b$. Since $\tau_b$ does not affect the forward output, it does not affect the selection gap. Temperature and gap are decoupled.

This decoupling enables adaptive scheduling: high $\tau_b$ early for exploration, low $\tau_b$ late for exploitation, without affecting alignment. Section~\ref{sec:cage} develops this into the CAGE algorithm.

%% file: section4_theory.tex
\section{Theoretical Analysis}
\label{sec:theory}

Section~\ref{sec:framework} decomposed the training-inference gap into selection gap (method-dependent) and computation gap (input-dependent). We now provide rigorous bounds on these gaps, analyze why different methods produce different solutions, and characterize how temperature affects convergence. Proofs appear in Appendix~\ref{app:proofs}. 

\subsection{Node-Level Gap Bounds}
\label{sec:node_level_bounds}

Consider a selection node choosing among $K$ gates with outputs $\{g_0, \ldots, g_{K-1}\}$ based on logits $\bz$. Let $i^* = \arg\max_i z_i$ denote the inference-time selection. We analyze the selection gap $\Delta h_{\text{sel}}$ from Section~\ref{sec:gap_decomposition}, which captures mismatch due to different selection mechanisms.

\begin{proposition}[Gumbel-Max Theorem {\citep{gumbel1954statistical}}]
\label{prop:gumbel_max}
For $G_i \stackrel{\mathrm{iid}}{\sim} \mathrm{Gumbel}(0,1)$:
\begin{equation}
    \Prob\bigl(\arg\max_i(z_i + G_i) = j\bigr) = \mathrm{softmax}(\bz)_j.
\end{equation}
This probability is temperature-independent since $\arg\max$ is scale-invariant. (Proof sketch in Appendix~\ref{proof:gumbel_max})
\end{proposition}

\begin{corollary}[Gumbel-ST: Selection Gap]
\label{cor:gumbel_st_gap}
For Gumbel-ST with $k = \arg\max_i(z_i + G_i)$, the expected selection gap is
\begin{equation}
    \E[\Delta h_{\text{sel}}] = \sum_{i \neq i^*} p_i (g_i - g_{i^*}), \quad p_i = \mathrm{softmax}(\bz)_i.
\end{equation}
This gap is determined entirely by the logit-induced distribution, independent of any temperature used in gradient computation. (Proof in Appendix~\ref{proof:gumbel_st_gap})
\end{corollary}

\begin{proposition}[Hard-ST: Zero Selection Gap]
\label{prop:hard_st_gap}
For Hard-ST, $\Delta h_{\text{sel}} = 0$ exactly, since both training and inference output $g_{i^*}$ where $i^* = \arg\max_i z_i$.  
\end{proposition}

\begin{proposition}[Mixture Methods: Selection Gap]
\label{prop:soft_gap}
For Soft-Mix and Soft-Gumbel, the selection gap is
\begin{equation}
    \Delta h_{\text{sel}} = \sum_{i \neq i^*} w_i(\tau) (g_i - g_{i^*}),
\end{equation}
where $w_i = \pi_i(\tau)$ for Soft-Mix or $w_i = \tilde{\pi}_i(\tau)$ for Soft-Gumbel. This gap depends on temperature $\tau$ through the weights. (Proof in Appendix~\ref{proof:soft_gap})
\end{proposition}

\begin{theorem}[Forward Alignment Principle]
\label{thm:forward_alignment}
Selection gap depends on whether training-time selection matches inference:
\begin{enumerate}[label=(\roman*), nosep, leftmargin=*]
    \item Hard-ST uses identical selection at training and inference, achieving $\Delta h_{\text{sel}} = 0$ exactly.
    \item Gumbel-ST samples a gate that matches $i^*$ with probability $p_{i^*}$, so $\E[\Delta h_{\text{sel}}] \to 0$ as $p_{i^*} \to 1$.
    \item Mixture methods compute weighted averages rather than selecting, yielding $\Delta h_{\text{sel}} \neq 0$ whenever gates produce distinct outputs.
\end{enumerate}
\end{theorem}

\begin{theorem}[Unified Selection Gap Bound]
\label{thm:gap_bound}
For any method with effective weights $w_i \geq 0$, $\sum_i w_i = 1$:
\begin{equation}
    |\Delta h_{\text{sel}}| \leq (1 - w_{i^*}) \cdot \Delta_{\max},
\end{equation}
where $\Delta_{\max} = \max_i |g_i - g_{i^*}|$. (Proof in Appendix~\ref{proof:gap_bound})
\end{theorem}

For Hard-ST, $w_{i^*} = 1$ yields zero gap. For Gumbel-ST, $w_{i^*} = p_{i^*}$ yields a bound of $(1-p_{i^*})\Delta_{\max}$ that tightens as confidence increases. For mixture methods, the bound depends on temperature through $w_{i^*}(\tau)$.

\subsection{Network-Level Analysis}
\label{sec:network_level_bounds}

Networks contain thousands of selection nodes. For the GroupSum architecture~\citep{petersen2022deep, yousefi2025mindthegap} where class logit $y_c = \sum_{n \in G_c} h_n$, the logit gap decomposes linearly: $\Delta y_c = \sum_{n \in G_c} \Delta h_n$.

\begin{proposition}[Loss Gap Bound]
\label{prop:holder}
Let $\by^{M}$ denote the network output logits using method $M$'s forward pass, and $\by^{\text{hard}}$ the output using hard selection with hard gate evaluation. For loss $\mathcal{L}(\by)$ with logit gap $\bDelta = \by^{M} - \by^{\text{hard}}$:
\begin{equation}
    |\Delta \mathcal{L}| \leq \|\nabla_{\by} \mathcal{L}\|_p \cdot \|\bDelta\|_q + O(\|\bDelta\|^2).
\end{equation}
(Proof in Appendix~\ref{proof:holder})
\end{proposition}

This bound shows that loss gap is controlled by prediction uncertainty (gradient norm) and total selection gap (logit gap). When nodes fully commit, the logit gap vanishes and so does the loss gap. However, the H\"{o}lder-based bound~\citep{rudin1987real} is worst-case and may be loose in practice, as individual node gaps can partially cancel. We therefore evaluate actual network-level gaps empirically in Section~\ref{sec:experiments}.

For Hard-ST, selection gap is exactly zero at every node, so the network-level selection gap is exactly zero by construction.

\subsection{Gradient Analysis}
\label{sec:gradients}

The gap analysis establishes \emph{what} differs at inference. We now analyze gradients to understand \emph{how} training dynamics differ.

\begin{proposition}[Gradient Structure]
\label{prop:unified_gradient}
All four methods compute logit gradients as:
\begin{equation}
    \frac{\partial \mathcal{L}}{\partial z_j} = \frac{\delta}{\tau_b} \cdot w_j \cdot (g_j - \bar{h}),
\end{equation}
where $w_j = \mathrm{softmax}(\bz/\tau_b)_j$, $\bar{h} = \sum_i w_i g_i$ is the surrogate mixture, and $\delta = \partial \mathcal{L}/\partial h$ is the upstream gradient evaluated at the method's forward output. For hard selection methods, argmax is non-differentiable; the straight-through estimator uses the soft surrogate for gradient computation while preserving $\delta$ from the hard forward output. (Proof in Appendix~\ref{proof:unified_gradient})
\end{proposition}

Although all methods share the same gradient form, they differ in where $\delta = \partial \mathcal{L}/\partial h$ is evaluated. For mixture methods, the loss is computed on $h = \bar{h}$, coupling $\delta$ to the soft mixture. For hard selection methods, the loss is computed on $h = g_k$ (the selected gate), so $\delta$ reflects the discrete output used at inference.

This distinction determines low-temperature behavior. For mixture methods, forward and backward share the same temperature $\tau$. As $\tau \to 0$, the mixture collapses ($\bar{h} \to g_{i^*}$), reducing selection gap structurally. However, gradients also vanish: the winner's $(g_{i^*} - \bar{h}) \to 0$; non-winners' $w_j \to 0$. This creates a dilemma: low $\tau$ reduces gap but cripples training.

Hard selection methods avoid this dilemma through temperature decoupling. The forward pass is $\tau$-independent, so selection gap depends on learned confidence, not temperature. The backward temperature $\tau_b$ controls only gradient distribution. However, $\tau_b$ must still be set appropriately: if too low, gradients collapse, training fails, and selection gap remains large (as observed with Gumbel-ST at low temperature). This motivates CAGE (Section~\ref{sec:cage}), which adapts $\tau_b$ based on training progress.

\begin{proposition}[Competitive Gate Selection]
\label{prop:competition}
Unlike standard neurons where parameters update independently, selection nodes induce competition via the gradient term $(g_j - \bar{h})$ from Proposition~\ref{prop:unified_gradient}. Gates outperforming the current mixture receive increased weight; underperforming gates are suppressed. This competition drives weight concentration over training. (Analysis in Appendix~\ref{proof:competition})
\end{proposition}

The competition leads to different outcomes. Mixture methods have achievable output space $[0,1]$, so hedging (diffuse weights producing $\bar{h} \notin \{0,1\}$) can reduce loss even if no single gate performs well. Hard methods have achievable output space $\{0,1\}$, so only the selected gate affects loss, forcing commitment. A node that hedges produces selection gap at inference when $\bar{h} \neq g_{i^*}$. Hard methods cannot hedge, eliminating this gap.

We verify this empirically in Section~\ref{sec:experiments}, measuring gate confidence $\max_i \pi_i$ across training.

\subsection{Convergence and Temperature}
\label{sec:convergence_temp}

\begin{proposition}[Convergence Scaling]
\label{prop:convergence}
For a single selection node with learning rate $\eta$ and backward temperature $\tau_b$, assuming approximately constant upstream gradients, steps to reach confidence $1-\epsilon$ scale as $T(\tau_b) = O(\tau_b \cdot \log(1/\epsilon) / \eta)$. (Proof in Appendix~\ref{proof:convergence})
\end{proposition}

Convergence time is linear in $\tau_b$. High $\tau_b$ produces diffuse weights $w_j$, distributing gradient across all gates and slowing winner emergence. Low $\tau_b$ concentrates gradient on the current winner, accelerating convergence but risking commitment to suboptimal gates before sufficient exploration.

This suggests adaptive scheduling: high $\tau_b$ early for broad exploration, decreasing over training for faster convergence. For hard selection methods, such adaptation affects only optimization dynamics; Proposition~\ref{prop:hard_st_gap} guarantees selection gap remains zero regardless of $\tau_b$ schedule.

%% file: section5_cage.tex
\section{CAGE: Confidence-Adaptive Gradient Estimation}
\label{sec:cage}

Corollary~\ref{cor:gumbel_st_gap} and Proposition ~\ref{prop:hard_st_gap} establish that hard selection methods decouple backward temperature $\tau_b$ from selection gap. Proposition~\ref{prop:convergence} shows convergence time scales linearly with $\tau_b$. This creates an opportunity: adapt $\tau_b$ throughout training to accelerate convergence as confidence increases, without affecting forward alignment.

\subsection{Exploration Mechanisms}

Hard selection methods offer two distinct exploration mechanisms.

\textbf{Forward exploration.} Gumbel-ST selects gate $i$ with probability $p_i = \mathrm{softmax}(\bz)_i$, occasionally selecting non-argmax gates. This is controlled by logit magnitudes, not temperature. Hard-ST has no forward exploration (deterministic argmax).

\textbf{Backward exploration.} The gradient to logit $z_j$ includes factor $\pi_j(\tau_b)$ from the ST surrogate. High $\tau_b$ distributes gradients across all gates; low $\tau_b$ concentrates gradients on the current winner. This is controlled by $\tau_b$ and can be adapted without affecting selection gap.

\subsection{The CAGE Algorithm}

CAGE adapts $\tau_b$ based on network-wide selection confidence:
\begin{equation}
    c(t) = \frac{1}{N} \sum_{n=1}^N \max_i \mathrm{softmax}(\bz_n(t))_i,
\end{equation}
where $N$ is the number of selection nodes and $t$ is the training step. At initialization, $c \approx 1/K$; as nodes commit, $c \to 1$.

CAGE maps confidence to temperature via linear interpolation:
\begin{equation}
    \tau_b(t) = \tau_{\max} - (\tau_{\max} - \tau_{\min}) \cdot \frac{c_{\text{ema}}(t) - 1/K}{1 - 1/K},
\end{equation}
where $c_{\text{ema}}$ is an exponential moving average~\citep{kingma2015adam} with coefficient $\beta$. Low confidence yields high $\tau_b$ for broad gradient distribution; high confidence yields low $\tau_b$ for faster convergence. Full pseudocode is provided in Algorithm~\ref{alg:cage} (Appendix~\ref{app:cage_algorithm}).

CAGE applies to both Hard-ST and Gumbel-ST. With Hard-ST, selection gap is 
exactly zero and exploration occurs purely through adaptive gradient distribution. With Gumbel-ST, low backward temperature causes gradient concentration on randomly sampled gates, leading to training collapse (Section~\ref{sec:gumbel_failure}). CAGE prevents this by maintaining high $\tau_b$ while confidence is low, only reducing temperature as gates commit. Forward stochasticity provides additional exploration while the selection gap remains small. Computational overhead is negligible: $\max_i \mathrm{softmax}(\bz_n)_i$ reuses the softmax computed for gradient estimation.

%% file: section6_experiments.tex

\section{Experiments}
\label{sec:experiments}

\subsection{Setup}
\label{sec:exp_setup}

We validate forward alignment by measuring selection gap across temperatures, convergence speed, gate commitment, and the effect of backward temperature adaptation via CAGE.

\textbf{Methods.}
We compare six methods: four from our $2 \times 2$ factorial (Soft-Mix, Soft-Gumbel, Hard-ST, Gumbel-ST) plus Hard-ST+CAGE and Gumbel-ST+CAGE (Section~\ref{sec:cage}).

\textbf{Datasets.}
We use three datasets that isolate different gap components.
\textbf{MNIST}~\citep{lecun1998gradient}: Continuous pixel values in $[0,1]$; both selection gap and computation gap are observable.
\textbf{MNIST-Binary}: Pixels thresholded to $\{0,1\}$; computation gap vanishes since soft gates equal hard gates on binary inputs, isolating \emph{pure selection gap}.
\textbf{CIFAR-10-Binary}~\citep{krizhevsky2009learning}: Each input channel is binarized using 31 evenly-spaced thresholds over $[0, 255]$; binary inputs isolate pure selection gap on a harder benchmark (3-channel natural images).

\textbf{Architecture and Training.}
We use flat (equal-width) logic gate networks with $L \in \{4, 5, 6\}$ layers, $K=16$ candidate gates per node, and 64k neurons per layer (128k for CIFAR-10). Training uses Adam optimizer with batch size 512. Learning rate is $10^{-2}$ for MNIST (50k iterations) and $5 \times 10^{-2}$ for CIFAR-10 (60k iterations). We sweep temperatures $\tau \in \{0.05, 0.1, 0.5, 1.0, 2.0\}$ with 3 random seeds per configuration. CAGE uses $\tau_{\max}=3.0$, $\tau_{\min}=0.5$, $\beta=0.99$. We evaluate training accuracy, test accuracy, and selection gap every 500 iterations. For stable training, we derive a variance-based temperature for the GroupSum layer that produces well-conditioned softmax outputs at initialization (Appendix~\ref{app:groupsum_tau}). Full hyperparameters are provided in Appendix~\ref{app:hyperparams}.

\subsection{Forward Alignment Eliminates Selection Gap}
\label{sec:results_gap}

Our central claim is that forward-pass structure determines the selection gap. 
Table~\ref{tab:selection_gap_full} tests this on MNIST-Binary ($L$=6), where 
binary inputs eliminate computation gap, isolating pure selection gap. We 
report final gap (at convergence) and peak gap (maximum during training).

The results reveal clear patterns supporting forward alignment:

\textbf{Hard-ST achieves exactly zero selection gap} at all temperatures, validating Proposition~\ref{prop:hard_st_gap}. This is the central result: because Hard-ST uses argmax in the forward pass during both training and inference, there is no selection mismatch by construction. The zero gap holds from $\tau=0.05$ to $\tau=2.0$, confirming this is a fundamental property of forward alignment, not an artifact of temperature choice.

\textbf{Temperature decoupling holds for hard methods.} Both Hard-ST (0\% everywhere) and Gumbel-ST (0--1.1\% at $\tau \geq 0.5$) achieve near-zero gap when training succeeds. The presence of Gumbel noise in the backward pass does not affect the gap for hard forward methods, confirming that backward-pass noise is decoupled from forward-pass alignment. This is a key prediction of our theoretical framework that distinguishes forward alignment from Hessian smoothing.

\textbf{Soft methods show temperature-dependent gaps.} Soft-Mix exhibits near-zero gap at low temperatures but develops significant gap at $\tau=2.0$ (0.6\% final, 24\% peak). Soft-Gumbel shows gaps at all temperatures (0.2--2.1\% final, 2--40\% peak), with Gumbel noise amplifying the hedging behavior. The consistent pattern (Soft-Gumbel $>$ Soft-Mix at every temperature) confirms that Gumbel noise encourages hedging by preventing commitment to individual gates.

\textbf{Gumbel-ST fails catastrophically at low temperature.} At $\tau \leq 0.1$, Gumbel-ST shows large negative gaps ($-40\%$ to $-59\%$), indicating training accuracy far below deployment accuracy. This counterintuitive result (deployment outperforming training) occurs because stochastic gate sampling during training prevents learning, while deterministic deployment benefits from whatever structure was learned. This is a training failure, not a gap issue, analyzed in Section~\ref{sec:gumbel_failure}.

Figure~\ref{fig:main_results}(a) visualizes the selection gap dynamics over training iterations. Hard-ST maintains exactly zero gap throughout training (blue line perfectly at 0\%), demonstrating that forward alignment provides immediate and sustained elimination of selection mismatch. Soft-Gumbel peaks at $\sim$17\% around iteration 5,000 before gradually recovering as the network commits to gates, but this recovery is incomplete, leaving residual gap at convergence. Interestingly, at low temperatures, Gumbel-ST shows transient \emph{negative} gaps early in training, reflecting periods where stochastic sampling disrupts learning (see Section~\ref{sec:gumbel_failure}). These patterns hold across all temperatures and layers; see Figure~\ref{fig:heatmap_app} in the appendix for gap heatmaps over the full (layer $\times$ $\tau$) grid.

\begin{table}[t]
\centering
\caption{Selection gap (\%) on MNIST-Binary ($L$=6). Format: Final / Peak (omitted for training failures marked with $\dagger$). Hard-ST achieves exactly zero gap at all temperatures. Soft methods show temperature-dependent gaps. Negative values indicate training failure (training accuracy below deployment). \textbf{Bold}: zero selection gap.}
\label{tab:selection_gap_full}
\small
\setlength{\tabcolsep}{3pt}
\begin{tabular}{lccccc}
\toprule
\textbf{Method} & $\tau$=0.05 & $\tau$=0.1 & $\tau$=0.5 & $\tau$=1.0 & $\tau$=2.0 \\
\midrule
Soft-Mix & 0.0 / 0 & 0.0 / 0 & 0.0 / 0.1 & 0.0 / 0.1 & {0.6 / 24} \\
Soft-Gumbel & 0.7 / 2 & 0.2 / 2 & 0.2 / 7 & 0.2 / 17 & {2.1 / 40} \\
\midrule
Hard-ST & \textbf{0 / 0} & \textbf{0 / 0} & \textbf{0 / 0} & \textbf{0 / 0} & \textbf{0 / 0} \\
Gumbel-ST & $-$40$^\dagger$ & $-$59$^\dagger$ & 0.0 / 0.1 & 0.0 / 0.3 & 0.0 / 1.1 \\
\midrule
Hard-ST+\scriptsize{CAGE} & 0 / 0 & 0 / 0 & 0 / 0 & 0 / 0 & 0 / 0 \\
Gumbel-ST+\scriptsize{CAGE} & 0.0 / 4.5 & 0.1 / 9 & 0.0 / 3.5 & 0.0 / 1.6 & 0.0 / 2.2 \\
\bottomrule
\end{tabular}
\\[1mm]
\scriptsize{$^\dagger$Training failure: accuracy collapses, deployment outperforms training.}
\end{table}

\begin{table}[t]
\centering
\caption{Temperature robustness: worst accuracy / range (\%) across $\tau \in \{0.05, 0.1, 0.5, 1.0, 2.0\}$, $L$=6. Gumbel-ST suffers catastrophic failure at low $\tau$ across all datasets; CAGE eliminates this sensitivity entirely. \textbf{Bold}: worst performance (failure case); \underline{underline}: best range (most robust).}
\label{tab:robustness}
\small
\setlength{\tabcolsep}{4pt}
\begin{tabular}{lccc}
\toprule
\textbf{Method} & \textbf{MNIST} & \textbf{M-Binary} & \textbf{C-10-Binary} \\
\midrule
Soft-Mix & 97.1 / 1.4 & 97.2 / 1.1 & 50.7 / 7.5 \\
Soft-Gumbel & 85.9 / 12.0 & 85.6 / 12.2 & 52.6 / 5.4 \\
\midrule
Hard-ST & 97.1 / 1.3 & 97.1 / 1.1 & 50.9 / 7.2 \\
Gumbel-ST & \textbf{52.1 / 46.3} & \textbf{51.0 / 47.3} & \textbf{18.9 / 39.5} \\
\midrule
Hard-ST+\scriptsize{CAGE} & 98.4 / \underline{0.1} & 98.1 / \underline{0.0} & 57.6 / \underline{0.5} \\
Gumbel-ST+\scriptsize{CAGE} & 98.4 / \underline{0.1} & 98.2 / \underline{0.1} & 58.0 / \underline{0.6} \\
\bottomrule
\end{tabular} \vspace{-1.0em}
\end{table}

\subsection{Gumbel-ST Failure at Low Temperature}
\label{sec:gumbel_failure}

Figure~\ref{fig:main_results}(b) reveals a critical failure mode distinct from the selection gap: Gumbel-ST accuracy collapses from $\sim$98\% ($\tau \geq 0.5$) to $\sim$50\% ($\tau=0.05$), reducing performance to near-random guessing on MNIST. Soft-Gumbel also suffers, dropping to $\sim$85\% at $\tau=0.05$, though less severely. In contrast, Hard-ST and all CAGE methods remain stable at 97--98.5\% across the entire temperature range.

This failure occurs because low $\tau$ concentrates the surrogate gradient on the sampled gate. When Gumbel noise selects a suboptimal gate early in training, the concentrated gradient reinforces that poor choice, creating a self-reinforcing loop that prevents recovery. Hard-ST avoids this via deterministic argmax selection: the same gate receives gradient updates in every forward pass, so learning signal consistently improves the best available option rather than reinforcing random early choices.

Table~\ref{tab:robustness} confirms this pattern across all three datasets. Gumbel-ST exhibits catastrophic accuracy variation (39--47\% range across temperatures), while CAGE variants remain stable (0.0--0.6\% range). On CIFAR-10-Binary, Gumbel-ST drops to just 18.9\%. CAGE prevents this failure by maintaining high temperature until gates commit.

\begin{figure*}[t]
    \centering
    \begin{subfigure}[t]{0.24\textwidth}
        \centering
        \includegraphics[width=\textwidth]{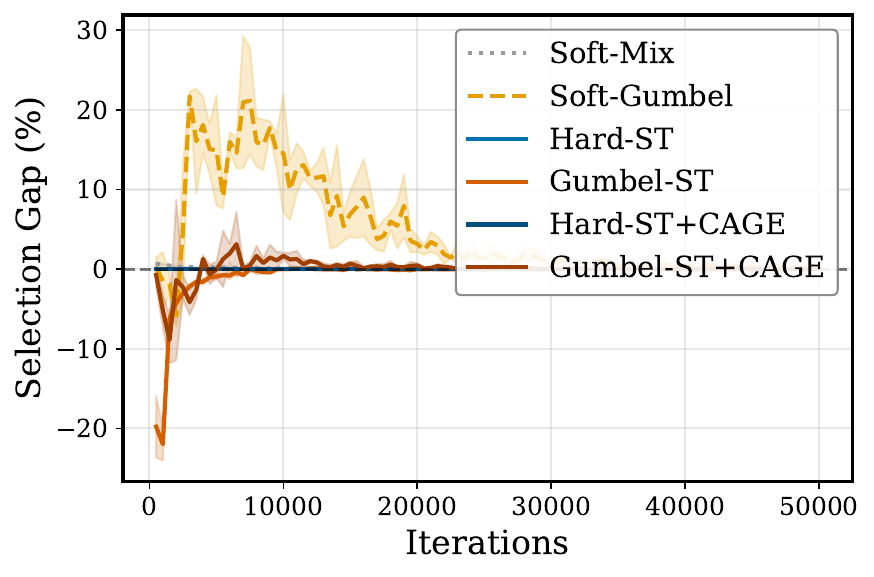}
        \caption{Selection gap}
        \label{fig:selection_gap}
    \end{subfigure}%
    \hfill
    \begin{subfigure}[t]{0.24\textwidth}
        \centering
        \includegraphics[width=\textwidth]{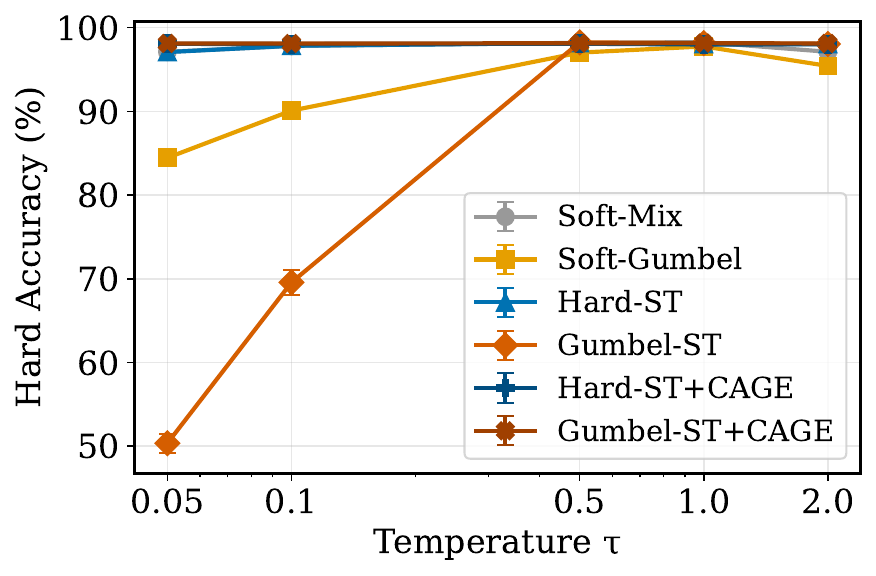}
        \caption{Accuracy vs $\tau$}
        \label{fig:accuracy_temp}
    \end{subfigure}%
    \hfill
    \begin{subfigure}[t]{0.24\textwidth}
        \centering
        \includegraphics[width=\textwidth]{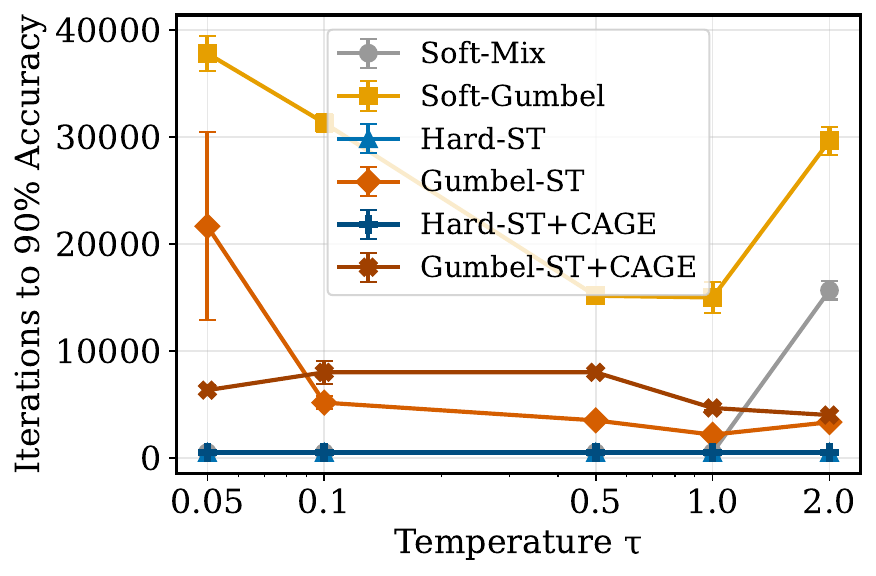}
        \caption{Convergence speed}
        \label{fig:convergence}
    \end{subfigure}%
    \hfill
    \begin{subfigure}[t]{0.24\textwidth}
        \centering
        \includegraphics[width=\textwidth]{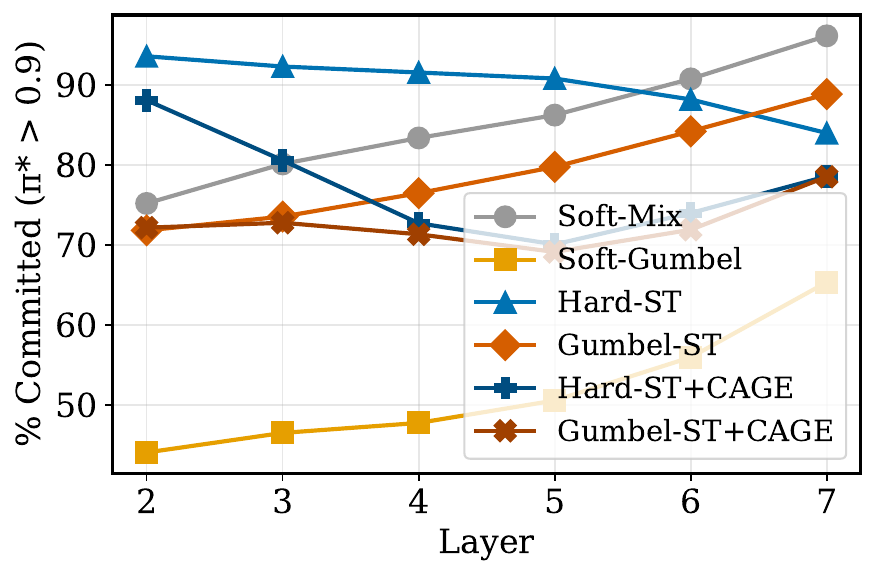}
        \caption{Gate commitment}
        \label{fig:gate_commitment}
    \end{subfigure}
    \caption{Main results (MNIST-Binary, $L$=6, $\tau$=1.0 unless varied). 
    (a) Selection gap during training: Hard-ST achieves exactly zero; soft methods show non-zero gaps.
    (b) Test accuracy vs temperature: Hard-ST and CAGE remain stable; Gumbel-ST fails at low $\tau$.
    (c) Iterations to convergence: Hard-ST converges significantly faster than soft methods.
    (d) Gate commitment by layer: Hard-ST commits; Soft-Gumbel hedges, especially in early layers.}
\label{fig:main_results}\vspace{-0.5em}
\end{figure*}

\subsection{Convergence Speed}
\label{sec:convergence}

Figure~\ref{fig:main_results}(c) reveals a marked difference in convergence speed: Hard-ST converges in approximately 500 iterations and maintains this speed across all temperatures, appearing as a flat line near the bottom of the plot. In contrast, Soft-Gumbel requires 15,000--38,000 iterations depending on temperature, making Hard-ST 30--76$\times$ faster. This directly contradicts claims that Gumbel noise accelerates training via Hessian regularization~\citep{yousefi2025mindthegap}.

The speed advantage stems from Hard-ST's gradient structure: deterministic argmax selection concentrates all learning signal on improving the current best gate, enabling rapid commitment. Gumbel methods dilute gradients across stochastically sampled gates, introducing variance that slows convergence. Notably, Hard-ST's convergence time is temperature-invariant ($\sim$500 iterations at any $\tau$), reflecting temperature decoupling. Soft-Gumbel shows a U-shaped pattern: slowest at extreme temperatures ($\sim$38,000 at $\tau=0.05$, $\sim$30,000 at $\tau=2.0$) and fastest at moderate temperatures.

\subsection{CAGE Eliminates Temperature Sensitivity}
\label{sec:cage_results}

Figure~\ref{fig:cage_adaptation} shows the dynamics of CAGE adaptation over 50,000 training iterations. In panel (a), the backward temperature $\tau_b$ starts at $\sim$2.4 (near $\tau_{\max}=3.0$) to encourage exploration, then decreases as gate confidence increases. Hard-ST+CAGE drops rapidly to $\sim$1.0 within 5,000 iterations and continues decreasing to $\sim$0.7 by convergence; Gumbel-ST+CAGE decreases more gradually, reaching $\sim$0.8. Notably, neither method reaches $\tau_{\min}=0.5$, indicating that the networks maintain some gradient smoothing throughout training rather than converging to hard gradients.

Panel (b) shows the corresponding confidence evolution. Both methods start at low confidence ($\sim$0.25) and increase as gates commit. Hard-ST+CAGE reaches $\sim$0.9 confidence by iteration 50,000, while Gumbel-ST+CAGE reaches $\sim$0.85. The faster confidence growth in Hard-ST+CAGE reflects deterministic selection's tendency toward rapid commitment.

CAGE proves highly effective at eliminating temperature sensitivity. On MNIST-Binary, CAGE reduces Gumbel-ST's accuracy range from 47.3 percentage points to just 0.1 percentage points (Table~\ref{tab:robustness}). On CIFAR-10-Binary, where Gumbel-ST collapses to 18.9\%, Gumbel-ST+CAGE maintains 58.0\% accuracy with only 0.6 percentage points variation—slightly outperforming Hard-ST+CAGE (57.6\%), suggesting that Gumbel noise provides beneficial exploration when stabilized by CAGE.

CAGE provides different benefits depending on the base method. For Hard-ST, improvement is modest (already robust due to temperature decoupling). For Gumbel-ST, CAGE is critical: it reduces accuracy variation from 39--47 percentage points to under 1 percentage point across all datasets (Table~\ref{tab:robustness}).

\begin{figure}[t]
    \centering
    \includegraphics[width=\linewidth]{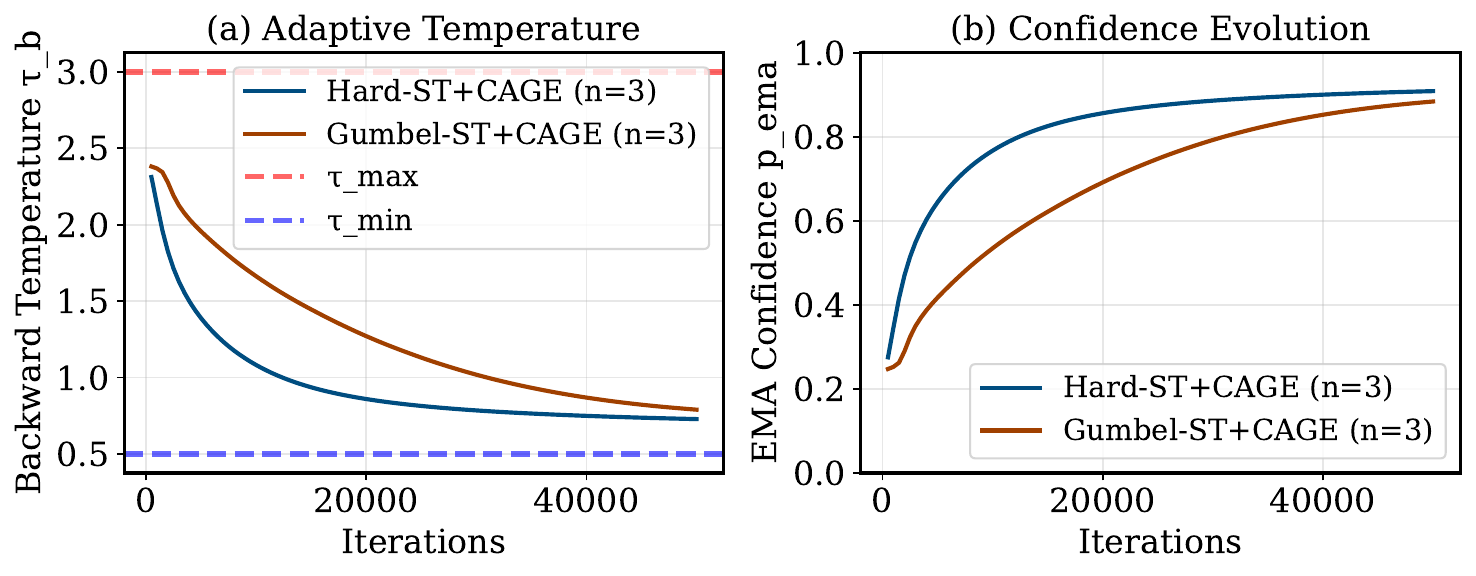} 
    \caption{CAGE adaptation dynamics (MNIST-Binary, $L{=}6$). (a) Backward temperature over training: $\tau_b$ decreases from $\tau_{\max}$ as confidence rises. (b) EMA confidence over training: increases from uniform ($\sim$0.25) toward commitment ($\sim$0.9). Results averaged over $n{=}3$ random seeds.}
\label{fig:cage_adaptation} \vspace{-0.9em}  
\end{figure}

\subsection{Gate Commitment Explains Hedging}
\label{sec:gate_commitment}

Proposition~\ref{prop:competition} predicts that mixture methods can hedge while hard selection methods must commit. Figure~\ref{fig:main_results}(d) validates this empirically: Hard-ST maintains high commitment (84--93\%) across all layers, while Soft-Gumbel shows markedly lower commitment (45--66\%). This difference directly explains the selection gap: Soft-Gumbel's hedging succeeds during training but fails at inference when argmax selects a single gate.

The layer-wise patterns reveal different commitment mechanisms. Hard-ST commits strongly from early layers, reflecting structural commitment forced by the hard forward pass. In contrast, Soft-Mix and Gumbel-ST show increasing commitment toward the output (72--96\%), driven by stronger gradient signal near the loss.

Interestingly, Soft-Mix achieves high commitment (75--96\%) despite soft forward computation. Without Gumbel noise, gradient flow concentrates on the dominant gate, and the network naturally commits. This suggests Gumbel noise, not soft forward computation alone, is the primary enabler of hedging.

Gumbel-ST shows intermediate commitment (72--90\%). Although each forward pass selects a single gate, stochastic sampling provides exploration, allowing more diversity than deterministic Hard-ST.

Hard-ST shows elevated XOR/XNOR usage ($\sim$11\% vs 6--8\% for other methods), suggesting deterministic selection encourages commitment to more expressive gates. Soft and Gumbel-based methods favor unitary gates (A, B, $\neg$A, $\neg$B), which pass specific inputs like skip connections, potentially explaining higher commitment in later layers. With CAGE, commitment decreases slightly, indicating that adaptive temperature encourages more exploration. See Appendix~\ref{app:gate_usage} for details. All findings are statistically significant ($p < 0.05$); see Appendix~\ref{app:statistical_tests}.

%% file: section9_conclusion.tex
\section{Conclusion}
\label{sec:conclusion}

We analyzed trainable discrete selection embedded in neural networks and confirmed forward alignment as the key principle for eliminating training-inference mismatch. This principle is simple but not trivial, as prior work has not clearly investigated the causes of gap reduction. Our results showed that Hard-ST achieved exactly zero discretization gap at all temperatures, while soft methods exhibited large peak gaps. Critically, Soft-Gumbel and Gumbel-ST used identical Gumbel noise but exhibited markedly different gaps: forward computation structure, not noise distribution, determined the gap. For practitioners, Hard-ST can be the default without careful tuning; Gumbel-ST with CAGE is recommended when exploration and stability are critical. While our experiments focused on logic gate networks, the principle may extend to other architectures requiring discrete selection at deployment (e.g., neural architecture search, mixture-of-experts, and vector quantization).

%% file: impact_statement.tex

\section*{Impact Statement}
This paper analyzes the training-inference gap in discrete selection networks and proposes methods to reduce it. Our contributions are primarily theoretical and methodological, focusing on understanding gradient estimation for discrete neural network components.

The techniques studied---logic gate networks, straight-through estimators, and temperature scheduling---are general-purpose tools applicable across many domains. We do not foresee specific negative societal consequences arising directly from this work. The potential applications (efficient neural networks, neural architecture search, mixture-of-experts routing) carry the same broad implications as machine learning research in general.

%% file: appendix.tex
\section{Potential Implications for Related Domains}
\label{app:implications}

Our theoretical framework may extend beyond logic gate networks. The forward alignment principle, which holds that training-time selection should match inference-time selection, predicts which methods avoid discretization gap. We briefly note potential connections; direct validation remains future work.

\paragraph{Neural Architecture Search.}
Differentiable NAS~\citep{liu2019darts} relaxes discrete architecture selection via softmax-weighted operation mixtures, enabling gradient-based optimization. However, this introduces a training-inference gap: the soft supernetwork used during search differs fundamentally from the discrete architecture deployed at test time. This gap manifests as various pathologies, including skip connection collapse where parameter-free operations dominate~\citep{chu2020fairdarts, zela2020understanding}, depth mismatch between search and evaluation~\citep{chen2019progressive}, and overfitting of architecture parameters~\citep{liang2019darts+}. Remedies include independent operation weighting~\citep{chu2020fairdarts}, perturbation-based regularization~\citep{zela2020understanding}, and Gumbel-softmax sampling~\citep{xie2019snas, dong2019gdas}. Our work addresses the analogous gap in logic gate networks, where the mismatch between real-valued relaxations and discrete Boolean operations causes similar performance degradation.

\paragraph{Mixture of Experts.}
Traditional MoE~\citep{jacobs1991adaptive,jordan1994hierarchical} uses soft gating, but modern large-scale models have largely converged on hard routing. Switch Transformer~\citep{fedus2021switch} uses hard top-1 (sparse) routing during both training and inference, avoiding soft-to-hard mismatch entirely. GShard~\citep{lepikhin2021gshard} and Expert Choice~\citep{zhou2022mixture} similarly employ hard routing with auxiliary losses for load balancing. Soft MoE~\citep{puigcerver2024soft} uses soft routing in both phases to avoid mismatch while accepting the cost of activating all experts. Recent work enables transitions between dense and sparse architectures~\citep{komatsuzaki2023sparse, zhang2022moefication}, converting pretrained dense models into MoE without retraining from scratch. Our framework provides one perspective on why hard routing dominates, though MoE faces distinct challenges (load balancing, expert utilization) orthogonal to selection gap.

\paragraph{Vector Quantization.}
VQ-VAE~\citep{oord2017neural, razavi2019generating} learns discrete representations via hard codebook selection (nearest neighbor lookup) with straight-through gradient estimation. The design was motivated by learning discrete latent structure and avoiding posterior collapse in models with powerful decoders. From our analysis perspective, this architecture naturally exhibits zero discretization gap: the forward pass uses identical hard selection during both training and inference. The commitment loss encourages encoder outputs to stay close to their assigned codebook entries, which can be interpreted as reducing the entropy of the implicit selection distribution—analogous to encouraging gate probabilities toward deterministic values in our framework.

\section{Training-Inference Gap Across Domains}
\label{app:domains}

Table~\ref{tab:pattern} summarizes the training-inference gap pattern across domains using discrete selection.

\begin{table}[h]
\centering
\small
\caption{Training-inference gap in discrete selection. Domains using soft training with hard inference suffer from mismatch; those using consistent selection avoid it.}
\label{tab:pattern}
\begin{tabular}{lllll}
\toprule
\textbf{Domain} & \textbf{Discrete Choices} & \textbf{Training} & \textbf{Inference} & \textbf{Gap?} \\
\midrule
\textbf{Logic Gate Networks}~\citep{petersen2022deep, kim2024narrowing} & Boolean gates & Soft mixture & Hard (argmax) & Yes \\
NAS (DARTS)~\citep{liu2019darts} & Operations & Soft mixture & Hard (argmax) & Yes \\
NAS (SNAS/GDAS)~\citep{xie2019snas, dong2019gdas} & Operations & Hard (sample) & Hard (argmax) & Small \\
MoE (traditional)~\citep{jacobs1991adaptive} & Expert networks & Soft gating & Soft gating & No \\
MoE (sparse)~\citep{fedus2021switch} & Expert networks & Hard (top-k) & Hard (top-k) & No \\
VQ-VAE~\citep{oord2017neural} & Codebook vectors & Hard (nearest) & Hard (nearest) & No \\
\bottomrule
\end{tabular}
\end{table}

%% file: appendix_proofs.tex

\section{Proofs of Theoretical Results}
\label{app:proofs}
This appendix provides complete proofs for all theoretical results stated in Section~\ref{sec:theory}. We use notation consistent with the main text: $h^M$ denotes method $M$'s forward output, $g_i$ denotes gate $i$'s output, and $\Delta h_{\text{sel}}$ denotes the selection gap.

\subsection{Proposition~\ref{prop:gumbel_max} (Gumbel-Max Theorem)}
\label{proof:gumbel_max}

This is a classical result~\citep{gumbel1954statistical}. We include a sketch for completeness.

\begin{proof}[Proof Sketch]
Let $G_i \stackrel{\mathrm{iid}}{\sim} \mathrm{Gumbel}(0,1)$ with CDF $F(x) = e^{-e^{-x}}$ and PDF $f(x) = e^{-x} e^{-e^{-x}}$. For option $j$ to win, we need $G_i < z_j - z_i + G_j$ for all $i \neq j$.

Conditioning on $G_j = g$ and applying the CDF:
$\Prob(G_i < z_j - z_i + g) = e^{-e^{-(z_j - z_i + g)}} = e^{-e^{-g} e^{z_i - z_j}}$.

By independence, the probability that $j$ wins is the product over all $i \neq j$:
\begin{equation}
    \Prob\bigl(\arg\max_i (z_i + G_i) = j \mid G_j = g\bigr) = \prod_{i \neq j} e^{-e^{-g} e^{z_i - z_j}} = \exp\Bigl(-e^{-g} \sum_{i \neq j} e^{z_i - z_j}\Bigr).
\end{equation}

Marginalizing over $G_j$ with PDF $f(g)$ and noting $1 + \sum_{i \neq j} e^{z_i - z_j} = \sum_i e^{z_i - z_j}$:
\begin{equation}
    \Prob\bigl(\arg\max_i (z_i + G_i) = j\bigr) = \int_{-\infty}^{\infty} e^{-g} \exp\Bigl(-e^{-g} \sum_{i} e^{z_i - z_j}\Bigr) dg.
\end{equation}

Substituting $u = e^{-g}$ gives
$\int_{0}^{\infty} e^{-u \sum_i e^{z_i - z_j}} du = \frac{1}{\sum_i e^{z_i - z_j}} = \frac{e^{z_j}}{\sum_i e^{z_i}} = \mathrm{softmax}(\bz)_j$. 

Scale invariance: $\arg\max_i v_i = \arg\max_i (v_i / c)$ for any $c > 0$.
\end{proof}

\subsection{Proof of Corollary~\ref{cor:gumbel_st_gap} (Gumbel-ST: Selection Gap)}
\label{proof:gumbel_st_gap}

\begin{proof}
Let $k = \arg\max_i(z_i + G_i)$ denote the stochastically selected gate. The training output is $h^M = g_k$.

By Proposition~\ref{prop:gumbel_max}, $\Prob(k = i) = p_i = \mathrm{softmax}(\bz)_i$. Taking expectations:
\begin{equation}
    \E[h^M] = \sum_{i=0}^{K-1} p_i g_i.
\end{equation}

The selection gap compares to $g_{i^*}$ where $i^* = \arg\max_i z_i$:
\begin{equation}
    \E[\Delta h_{\text{sel}}] = \sum_{i=0}^{K-1} p_i g_i - g_{i^*} = \sum_{i \neq i^*} p_i(g_i - g_{i^*}).
\end{equation}

Since $p_i = \mathrm{softmax}(\bz)_i$ is determined by raw logits (no temperature 
in the Gumbel-Max forward pass), the expected gap is independent of backward 
temperature $\tau_b$.
\end{proof}

\subsection{Proof of Proposition~\ref{prop:hard_st_gap} (Hard-ST: Zero Selection Gap)}
\label{proof:hard_st_gap}
\begin{proof}
Hard-ST uses $i^* = \arg\max_i z_i$ for both training and inference. 
Thus $h^M = g_{i^*} = h^{\text{hard}}$, yielding $\Delta h_{\text{sel}} = 0$ exactly.
\end{proof}

\subsection{Proof of Proposition~\ref{prop:soft_gap} (Mixture Methods: Selection Gap)}
\label{proof:soft_gap}

\begin{proof}
\textbf{Soft-Mix.} The training output is $h^M = \sum_i \pi_i(\tau) g_i$ where $\pi_i(\tau) = \mathrm{softmax}(\bz/\tau)_i$. The selection gap:
\begin{equation}
    \Delta h_{\text{sel}} = \sum_{i=0}^{K-1} \pi_i(\tau) g_i - g_{i^*} = \sum_{i \neq i^*} \pi_i(\tau) (g_i - g_{i^*}).
\end{equation}
Since $\pi_i(\tau)$ depends explicitly on $\tau$, the selection gap is temperature-dependent.

\textbf{Soft-Gumbel.} The training output is $h^M = \sum_i \tilde{\pi}_i(\tau) g_i$ where $\tilde{\pi}_i(\tau) = \mathrm{softmax}((\bz+\bG)/\tau)_i$. The selection gap has the same form:
\begin{equation}
    \Delta h_{\text{sel}} = \sum_{i \neq i^*} \tilde{\pi}_i(\tau) (g_i - g_{i^*}).
\end{equation}
Gumbel noise perturbs the weights but preserves the mixture structure, so the gap remains temperature-dependent.
\end{proof}

\subsection{Proof of Theorem~\ref{thm:gap_bound} (Unified Selection Gap Bound)}
\label{proof:gap_bound}

\begin{proof}
Starting from the selection gap with effective weights $\bw$:
\begin{equation}
    |\Delta h_{\text{sel}}| = \left|\sum_{i \neq i^*} w_i (g_i - g_{i^*})\right|.
\end{equation}

Applying the triangle inequality and bounding $|g_i - g_{i^*}| \leq \Delta_{\max}$:
\begin{equation}
    |\Delta h_{\text{sel}}| \leq \sum_{i \neq i^*} w_i |g_i - g_{i^*}| \leq \sum_{i \neq i^*} w_i \cdot \Delta_{\max} = (1 - w_{i^*}) \cdot \Delta_{\max}.
\end{equation}

The bound is tight when all non-winner weight concentrates on the gate with maximum deviation.
\end{proof}

\subsection{Proof of Proposition~\ref{prop:holder} (Loss Gap Bound)}
\label{proof:holder}

\begin{proof}
For GroupSum architecture, the class $c$ output is $y_c = \sum_{n \in G_c} h_n$. By linearity, the logit gap decomposes as $\Delta y_c = \sum_{n \in G_c} \Delta h_n$.

By first-order Taylor expansion around $\by^{\text{hard}}$:
\begin{equation}
    \mathcal{L}(\by^M) \approx \mathcal{L}(\by^{\text{hard}}) + \nabla_{\by} \mathcal{L}\big|_{\by^{\text{hard}}}^\top \bDelta + O(\|\bDelta\|^2),
\end{equation}
where $\bDelta = \by^M - \by^{\text{hard}}$.

Applying H\"{o}lder's inequality with $(p,q) = (1, \infty)$:
\begin{equation}
    |\nabla_{\by} \mathcal{L}^\top \bDelta| \leq \|\nabla_{\by} \mathcal{L}\|_1 \cdot \|\bDelta\|_\infty.
\end{equation}

For cross-entropy loss, let $\sigma_c = \mathrm{softmax}(\by^{\text{hard}})_c$ denote the predicted probability for class $c$ at inference. The gradient is:
\begin{equation}
    \frac{\partial \mathcal{L}}{\partial y_c} = \sigma_c - \mathbf{1}[c = y].
\end{equation}

The $L_1$ norm:
\begin{equation}
    \|\nabla_{\by} \mathcal{L}\|_1 = |1 - \sigma_y| + \sum_{c \neq y} \sigma_c = (1 - \sigma_y) + (1 - \sigma_y) = 2(1 - \sigma_y).
\end{equation}

For the logit gap, using the node-level bound (Theorem~\ref{thm:gap_bound}) with $\Delta_{\max} = 1$:
\begin{equation}
    \|\bDelta\|_\infty = \max_c |\Delta y_c| \leq \max_c \sum_{n \in G_c} |\Delta h_n| \leq \max_c \sum_{n \in G_c} (1 - w_{n,i^*_n}).
\end{equation}

Combining:
\begin{equation}
    |\Delta \mathcal{L}| \leq 2(1 - \sigma_y) \max_c \sum_{n \in G_c} (1 - w_{n,i^*_n}) + O(\|\bDelta\|^2).
\end{equation}
This bound shows that loss gap is controlled by the product of prediction uncertainty $(1 - \sigma_y)$ and total non-commitment $\sum_n (1 - w_{n,i^*_n})$. When nodes fully commit, the bound vanishes. For Hard-ST, selection gap is exactly zero at every node, so the bound is trivially satisfied.

\end{proof}

\subsection{Proof of Proposition~\ref{prop:unified_gradient} (Gradient Structure)}
\label{proof:unified_gradient}

\begin{proof}
For any method using surrogate weights $w_j(\tau)$, the gradient with respect to logit $z_j$ is:
\begin{equation}
    \frac{\partial \mathcal{L}}{\partial z_j} = \frac{\partial \mathcal{L}}{\partial h} \cdot \frac{\partial h}{\partial z_j}.
\end{equation}

The surrogate gradient (straight-through for hard methods, direct for soft methods) uses:
\begin{equation}
    \frac{\partial h}{\partial z_j} \approx \frac{\partial}{\partial z_j} \sum_i w_i(\tau) g_i = \sum_i g_i \frac{\partial w_i}{\partial z_j}.
\end{equation}

For softmax weights: $\frac{\partial w_i}{\partial z_j} = \frac{1}{\tau} w_i (\mathbf{1}[i=j] - w_j)$.

Substituting and simplifying:
\begin{equation}
    \frac{\partial \mathcal{L}}{\partial z_j} = \frac{\delta}{\tau} \cdot w_j(\tau) \cdot (g_j - \bar{h}),
\end{equation}
where $\delta = \partial \mathcal{L}/\partial h$ and $\bar{h} = \sum_i w_i g_i$.
\end{proof}

\paragraph{Where $\delta$ is evaluated.}
The gradient formula is identical across methods, but $\delta$ is evaluated at different points.
\textbf{Mixture methods:} $\delta = \partial\mathcal{L}/\partial h|_{h=\bar{h}}$; the gradient exactly optimizes $\mathcal{L}(\bar{h}, y)$.
\textbf{Hard-ST:} $\delta = \partial\mathcal{L}/\partial h|_{h=g_{i^*}}$; forward uses discrete output, backward uses soft surrogate.
\textbf{Gumbel-ST:} $\delta = \partial\mathcal{L}/\partial h|_{h=g_k}$ for sampled $k$; in expectation, approximates $\nabla_{\bz} \sum_i p_i \mathcal{L}(g_i, y)$.

\paragraph{Competition term.}
The term $(g_j - \bar{h})$ induces competition: gates with $g_j > \bar{h}$ (outperforming the mixture) receive positive gradient when $\delta < 0$, increasing their selection weight. This drives commitment over training.

\subsection{Analysis of Proposition~\ref{prop:competition} (Competitive Gate Selection)}
\label{proof:competition}

We analyze how the competition term $(g_j - \bar{h})$ leads to different behaviors across methods.

\paragraph{Output spaces.}
The achievable output space differs by method.
\textbf{Mixture methods:} $\mathcal{H}_{\mathrm{soft}} = \{\sum_i w_i g_i : \mathbf{w} \in \Delta^{K-1}\} = [0,1]$, the convex hull of gate outputs.
\textbf{Hard methods:} $\mathcal{H}_{\mathrm{hard}} = \{g_0, \ldots, g_{K-1}\} \subseteq \{0,1\}$, the discrete set of gate outputs.

\paragraph{Mixture methods can hedge.}
Since $\mathcal{H}_{\mathrm{soft}} = [0,1]$ is continuous, multiple weight configurations can achieve similar blended outputs. The optimizer can find favorable combinations where $\bar{h} = \sum_i w_i g_i$ reduces loss, even if no single gate $g_i$ performs well individually. This is hedging: relying on the blend rather than committing to one gate.

\paragraph{Hard methods must commit.}
Since $\mathcal{H}_{\mathrm{hard}} \subseteq \{0,1\}$ contains only discrete outputs, the optimizer must find a single gate that performs well. For Hard-ST, this requires $z_k > z_j$ for all $j \neq k$. For Gumbel-ST, reliable selection requires $p_k \approx 1$, meaning $z_k \gg z_j$. Both conditions demand strong logit separation, forcing commitment.

\paragraph{Connection to selection gap.}
Hedging solutions have $\bar{h} \notin \{0,1\}$, so $\bar{h} \neq g_{i^*}$ generically, producing selection gap. Commitment solutions have concentrated weights, so $\bar{h} \approx g_{i^*}$, eliminating gap.

\subsection{Proof of Proposition~\ref{prop:convergence} (Convergence Scaling)}
\label{proof:convergence}

\begin{proof}
Consider a single selection node choosing among $K$ options with logits 
$\mathbf{z} = (z_1, \ldots, z_K)$, learning rate $\eta$, and backward 
temperature $\tau_b$.

\textbf{Step 1: Softmax gradient scaling.}
The softmax probabilities with temperature are:
\begin{equation}
    p_i = \frac{e^{z_i/\tau_b}}{\sum_j e^{z_j/\tau_b}}.
\end{equation}
The gradient of $p_i$ with respect to $z_i$ is:
\begin{equation}
    \frac{\partial p_i}{\partial z_i} = \frac{1}{\tau_b} p_i(1 - p_i) = O(1/\tau_b).
\end{equation}
Thus, gradient magnitude scales inversely with $\tau_b$.

\textbf{Step 2: Required logit gap for target confidence.}
To achieve confidence $p_{i^*} \geq 1 - \epsilon$ for the dominant option $i^*$, 
we require:
\begin{equation}
    \frac{e^{z_{i^*}}}{\sum_j e^{z_j}} \geq 1 - \epsilon.
\end{equation}
Assuming other logits are approximately equal to some baseline $z_0$, let 
$\Delta = z_{i^*} - z_0$. Then:
\begin{align}
    \frac{e^{\Delta}}{e^{\Delta} + (K-1)} &\geq 1 - \epsilon \\
    e^{\Delta} &\geq (1-\epsilon)(e^{\Delta} + (K-1)) \\
    \epsilon \cdot e^{\Delta} &\geq (1-\epsilon)(K-1) \\
    \Delta &\geq \log\left(\frac{(1-\epsilon)(K-1)}{\epsilon}\right).
\end{align}
For small $\epsilon$ and fixed $K$:
\begin{equation}
    \Delta = O(\log(1/\epsilon)).
\end{equation}

\textbf{Step 3: Per-step logit update.}
With learning rate $\eta$, gradient magnitude $O(1/\tau_b)$, and approximately 
constant upstream gradient $g$, the per-step update to the dominant logit is:
\begin{equation}
    \Delta z_{i^*} = \eta \cdot \frac{g}{\tau_b} \cdot p_{i^*}(1-p_{i^*}) = O(\eta/\tau_b).
\end{equation}
Similarly, non-dominant logits decrease, so the gap $\Delta$ grows at rate 
$O(\eta/\tau_b)$ per step.

\textbf{Step 4: Convergence time.}
The number of steps to achieve the required gap is:
\begin{equation}
    T(\tau_b) = \frac{\text{Required gap}}{\text{Gap growth per step}} 
    = \frac{O(\log(1/\epsilon))}{O(\eta/\tau_b)} 
    = O\left(\frac{\tau_b \cdot \log(1/\epsilon)}{\eta}\right).
\end{equation}
\end{proof}

\section{CAGE Algorithm Details}
\label{app:cage_algorithm}
A pseudocode of the CAGE is shown in
Algorithm~\ref{alg:cage}.
\begin{algorithm}[htbp]
\caption{CAGE: Confidence-Adaptive Gradient Estimation}
\label{alg:cage}
\begin{algorithmic}
\REQUIRE $\tau_{\min}=0.5$, $\tau_{\max}=3.0$, $\beta=0.99$
\STATE Initialize $c_{\text{ema}} \leftarrow 1/K$
\FOR{each training step $t$}
    \STATE $c(t) \leftarrow \frac{1}{N}\sum_{n=1}^N \max_i \text{softmax}(\mathbf{z}_n)_i$
    \STATE $c_{\text{ema}} \leftarrow \beta \cdot c_{\text{ema}} + (1-\beta) \cdot c(t)$
    \STATE $\tau_b \leftarrow \tau_{\max} - (\tau_{\max} - \tau_{\min}) \cdot \frac{c_{\text{ema}} - 1/K}{1 - 1/K}$
    \STATE Forward: hard selection (Hard-ST or Gumbel-ST)
    \STATE Backward: ST gradients with temperature $\tau_b$
\ENDFOR
\end{algorithmic}
\end{algorithm}

\section{Critique of Hessian Regularization Hypothesis}
\label{app:hessian_critique}

Recent work by~\citet{yousefi2025mindthegap} analyzed the discretization gap through the lens of ``implicit Hessian regularization.'' Their analysis assumes the loss $\mathcal{L}$ is computed directly on softmax outputs, yielding an expansion:
\begin{align}
    J(\bz) &= \E_{\mathbf{G}}[\mathcal{L}(\mathrm{softmax}((\bz+\mathbf{G})/\tau))] \nonumber \\
    &= \mathcal{L}(\mathrm{softmax}(\bz/\tau)) + \frac{\pi^2}{12\tau^2}\mathrm{tr}(H_{\mathcal{L}}) + O(\tau^{-3}),
\end{align}
where the Hessian trace term provides implicit regularization. The hypothesis is that Gumbel noise reduces mismatch through this regularization effect.

We identify three issues:

\textbf{Mixture vs.\ Hard Forward.} The analysis assumes mixture forward. For Gumbel-ST, the forward output is discrete ($h^M = g_k$), and the Hessian expansion does not directly apply.

\textbf{Architecture Mismatch.} In deep networks, selection logits propagate through multiple layers before reaching the loss. The relationship between selection temperature and loss Hessian is more complex than the single-layer analysis suggests.

\textbf{Empirical Evidence.} Soft-Gumbel and Gumbel-ST use identical Gumbel noise but exhibit markedly different gaps. This suggests gap reduction stems from forward alignment, not Hessian regularization.

\section{Extended Gradient Analysis}
\label{app:gradients}

\paragraph{Gumbel Methods Maintain Gradient Flow.}
Even at small $\tau_b$, Gumbel noise can shift the winner: $k = \arg\max_i(z_i + G_i)$ differs from $i^* = \arg\max_i z_i$ with probability $1 - p_{i^*} > 0$.

When $k \neq i^*$, the output $h^M = g_k \neq g_{i^*}$ provides non-zero gradient signal. Over multiple training steps, different gates win, maintaining gradient flow to all gates. This exploration mechanism is distinct from the soft surrogate used in the backward pass.

\paragraph{Low-Temperature Failure Mode.}
For mixture methods at low $\tau$, gradients vanish for two reasons: (1) the winner's deviation $(g_{i^*} - \bar{h}) \to 0$ as $\bar{h} \to g_{i^*}$, and (2) non-winner weights $w_j \to 0$. Hard selection methods avoid this because $\delta$ is evaluated at the discrete output, not the collapsing mixture.
\section{GroupSum Temperature Derivation}
\label{app:groupsum_tau}

We derive the temperature $\tau$ for GroupSum layers~\citep{petersen2022deep, kim2023deep} that produces well-conditioned softmax outputs at initialization.

Consider classification with $C$ classes, $k$ neurons per class, each output $h_n \in [0,1]$, GroupSum readout $s_c = \sum_{n \in \mathcal{G}_c} h_n$, and scaled logits $s'_c = s_c / \tau$.

\textbf{Statistics at Initialization.}
With random gate selection, $h_n \sim \text{Bernoulli}(0.5)$, giving $\E[s_c] = 0.5k$ and $\text{Var}[s_c] = 0.25k$. Since softmax is translation-invariant, only logit differences matter. Let $\Delta s = s_i - s_j$ denote the difference between two class sums. Since $s_i$ and $s_j$ are independent:
\begin{equation}
    \text{Var}[\Delta s] = \text{Var}[s_i] + \text{Var}[s_j] = 0.5k, \quad \text{Std}[\Delta s] = \sqrt{0.5}\sqrt{k} \approx 0.71\sqrt{k}.
\end{equation}

\textbf{Optimal Temperature.}
We target $\sigma_{\max} \approx 0.66$ (a design choice balancing gradient flow: not too peaked, not too uniform). Solving $\sigma_{\max} = \exp(\Delta s / \tau) / [\exp(\Delta s / \tau) + (C-1)] = 0.66$ gives $\Delta s / \tau \approx \ln(C-1) + 0.7$. Using two standard deviations as typical maximum difference ($\Delta s \approx 1.42\sqrt{k}$):
\begin{equation}
    \tau = \frac{1.42\sqrt{k}}{\ln(C-1) + 0.7} = \alpha(C)\sqrt{k},
\end{equation}
where $\alpha(C) = 1.42 / (\ln(C-1) + 0.7)$. For common class counts: $\alpha(10) = 0.49$, $\alpha(100) = 0.27$, $\alpha(1000) = 0.19$.
 

%% file: appendix_experiments.tex

\section{Experimental Details}
\label{app:detailed_results}

\subsection{Hyperparameters}
\label{app:hyperparams}
Table~\ref{tab:hyperparams_app} summarizes the training configurations used in our experiments.

\begin{table}[h]
\centering
\caption{Training configuration.}
\label{tab:hyperparams_app}
\small
\begin{tabular}{lll}
\toprule
\textbf{Parameter} & \textbf{MNIST} & \textbf{CIFAR-10} \\
\midrule
Neurons per layer & 64,000 & 128,000 \\
Layers & 4, 5, 6 & 4, 5, 6 \\
Gates per node & 16 & 16 \\
Batch size & 512 & 512 \\
Learning rate & 0.01 & 0.05 \\
Optimizer & Adam & Adam \\
Iterations & 50,000 & 60,000 \\
Seeds & 3 & 3 \\
\midrule
\multicolumn{3}{l}{\textit{Temperature sweep}: $\tau \in \{0.05, 0.1, 0.5, 1.0, 2.0\}$} \\
\multicolumn{3}{l}{\textit{CAGE}: $\tau_{\max}=3.0$, $\tau_{\min}=0.5$, $\beta=0.99$} \\
\bottomrule
\end{tabular}
\end{table}

\subsection{Cross-Dataset Summary}
\label{app:cross_dataset_summary}

Table~\ref{tab:cross_dataset_summary} provides a consolidated view of key metrics across all datasets at $L$=6, the primary configuration reported in the main text.

Hard-ST achieves zero selection gap on MNIST-Binary and CIFAR-10-Binary (binary inputs eliminate computation gap). On MNIST, the small 0.4\% gap is purely computation gap from continuous inputs.
Gumbel-ST fails catastrophically at low $\tau$ across all datasets, with accuracy variation of 39--47 percentage points.
CAGE eliminates temperature sensitivity, reducing accuracy variation to $<$1 percentage point for both Hard-ST and Gumbel-ST.

\begin{table*}[htb]
\centering
\caption{Cross-dataset summary ($L$=6). Peak selection gap (\%) during training and accuracy range (\%) across $\tau \in \{0.05, 0.1, 0.5, 1.0, 2.0\}$. Hard-ST achieves zero selection gap and minimal accuracy range across all datasets. Gumbel-ST suffers catastrophic training failure at low $\tau$. \textbf{Bold}: zero selection gap; \underline{underline}: best accuracy range.}
\label{tab:cross_dataset_summary}
\small
\begin{tabular}{lcccccc}
\toprule
& \multicolumn{2}{c}{\textbf{MNIST}} & \multicolumn{2}{c}{\textbf{MNIST-Binary}} & \multicolumn{2}{c}{\textbf{CIFAR-10-Binary}} \\
\cmidrule(lr){2-3} \cmidrule(lr){4-5} \cmidrule(lr){6-7}
\textbf{Method} & Peak Gap & Acc Range & Peak Gap & Acc Range & Peak Gap & Acc Range \\
\midrule
Soft-Mix & 23\% & 1.4\% & 24\% & 1.1\% & 9\% & 7.5\% \\
Soft-Gumbel & 40\% & 12.0\% & 40\% & 12.2\% & 3.5\% & 5.4\% \\
\midrule
Hard-ST & 0.4\%$^*$ & 1.3\% & \textbf{0\%} & 1.1\% & \textbf{0\%} & 7.2\% \\
Gumbel-ST & 62\%$^\dagger$ & 46.3\% & 61\%$^\dagger$ & 47.3\% & 16\%$^\dagger$ & 39.5\% \\
\midrule
Hard-ST+CAGE & 0.3\%$^*$ & \underline{0.1\%} & \textbf{0\%} & \underline{0.0\%} & \textbf{0\%} & \underline{0.5\%} \\
Gumbel-ST+CAGE & 9\% & \underline{0.1\%} & 9\% & \underline{0.1\%} & 0.7\% & \underline{0.6\%} \\
\bottomrule
\end{tabular}
\\[1mm]
\scriptsize{$^*$Includes computation gap from continuous inputs (selection gap is zero). $^\dagger$Training failure at low $\tau$ produces large negative gap during training.}
\end{table*}

\subsection{Gap Dynamics}
\label{app:gap_dynamics}

\paragraph{MNIST.}
Table~\ref{tab:gap_dynamics_mnist} shows selection gap dynamics during training on MNIST. The format ``Final [min, max]'' indicates the final gap at convergence and the range observed during the last 80\% of training (excluding initial warm-up).

MNIST has continuous pixel values in $[0,1]$, so the measured gap includes both selection and computation components. Hard-ST shows $\sim$0.2\% gap from computation gap alone (selection gap is zero by construction). Gumbel-ST exhibits catastrophic negative gaps at $\tau \leq 0.1$, indicating training failure where deployment accuracy exceeds training accuracy.

\begin{table}[h]
\centering
\caption{Gap trajectory on MNIST. Format: Final [min, max] (\%). Range computed over last 80\% of training. Hard-ST shows stable near-zero gap; Soft-Gumbel shows large transient peaks that increase with $\tau$.}
\label{tab:gap_dynamics_mnist}
\scriptsize
\begin{tabular}{llccccc}
\toprule
\textbf{Method} & \textbf{L} & $\tau$=0.05 & $\tau$=0.1 & $\tau$=0.5 & $\tau$=1.0 & $\tau$=2.0 \\
\midrule
Soft-Mix & 4 & $+0.3$ [+0.2,+0.4] & $+0.2$ [+0.1,+0.3] & $+0.2$ [+0.1,+0.3] & $+0.1$ [+0.0,+0.3] & $+0.2$ [+0.2,+0.5] \\
 & 5 & $+0.3$ [+0.2,+0.4] & $+0.2$ [+0.1,+0.3] & $+0.2$ [+0.1,+0.3] & $+0.1$ [+0.1,+0.3] & $+0.3$ [+0.2,+1.5] \\
 & 6 & $+0.3$ [+0.1,+0.4] & $+0.2$ [+0.1,+0.3] & $+0.2$ [+0.1,+0.3] & $+0.1$ [+0.1,+0.3] & $+0.7$ [+0,+23] \\
\midrule
Soft-Gumbel & 4 & $+0.6$ [-31,+2] & $+0.2$ [-15,+2] & $+0.3$ [-0.1,+2.3] & $+0.3$ [+0,+10] & $+0.5$ [+0,+21] \\
 & 5 & $-0.2$ [-33,+1] & $+0.2$ [-33,+2] & $+0.3$ [+0.0,+2.9] & $+0.4$ [+0,+16] & $+1.1$ [+1,+33] \\
 & 6 & $+0.2$ [-29,+4] & $+0.6$ [-25,+4] & $+0.4$ [-0,+9] & $+0.4$ [+0,+16] & $+1.7$ [+1,+40] \\
\midrule
Hard-ST & 4 & $+0.2$ [+0.2,+0.4] & $+0.2$ [+0.2,+0.3] & $+0.3$ [+0.2,+0.4] & $+0.2$ [+0.1,+0.4] & $+0.2$ [+0.1,+0.4] \\
 & 5 & $+0.3$ [+0.2,+0.4] & $+0.2$ [+0.1,+0.3] & $+0.2$ [+0.1,+0.3] & $+0.2$ [+0.1,+0.3] & $+0.2$ [+0.1,+0.3] \\
 & 6 & $+0.2$ [+0.2,+0.4] & $+0.2$ [+0.1,+0.3] & $+0.2$ [+0.1,+0.2] & $+0.2$ [+0.1,+0.4] & $+0.1$ [+0.0,+0.3] \\
\midrule
Gumbel-ST & 4 & $-46.3$ [-48,-41] & $-43.2$ [-56,-42] & $+0.2$ [-0.2,+0.4] & $+0.1$ [+0.0,+0.6] & $+0.1$ [+0.0,+0.8] \\
 & 5 & $-41.8$ [-43,-28] & $-55.2$ [-59.0,-54.5] & $+0.1$ [-0.4,+0.2] & $+0.2$ [-0.1,+0.4] & $+0.2$ [+0.0,+0.9] \\
 & 6 & $-41.8$ [-43,-31] & $-58.2$ [-62,-55] & $+0.2$ [-0.8,+0.3] & $+0.2$ [-0.1,+0.4] & $+0.2$ [+0.0,+1.4] \\
\midrule
Hard-ST+CAGE & 4 & $+0.1$ [+0.0,+0.3] & $+0.2$ [+0.0,+0.3] & $+0.2$ [+0.0,+0.3] & $+0.2$ [+0.1,+0.3] & $+0.3$ [+0.1,+0.3] \\
 & 5 & $+0.2$ [+0.1,+0.3] & $+0.2$ [+0.1,+0.3] & $+0.2$ [+0.0,+0.3] & $+0.2$ [+0.1,+0.3] & $+0.1$ [+0.0,+0.3] \\
 & 6 & $+0.2$ [+0.0,+0.3] & $+0.2$ [+0.0,+0.3] & $+0.1$ [+0.0,+0.2] & $+0.2$ [+0.1,+0.3] & $+0.1$ [+0.0,+0.3] \\
\midrule
Gumbel-ST+CAGE & 4 & $+0.2$ [+0.0,+0.5] & $+0.3$ [+0.0,+0.7] & $+0.2$ [-0.1,+1.1] & $+0.2$ [+0.0,+0.5] & $+0.2$ [+0.0,+0.6] \\
 & 5 & $+0.2$ [+0.0,+0.9] & $+0.1$ [+0.0,+1.3] & $+0.2$ [+0.0,+1.8] & $+0.2$ [-0.1,+0.8] & $+0.1$ [+0.0,+0.9] \\
 & 6 & $+0.2$ [-0.1,+3.2] & $+0.2$ [-0,+5] & $+0.2$ [-0,+5] & $+0.2$ [-0.1,+2.6] & $+0.1$ [-0.2,+1.7] \\
\bottomrule
\end{tabular}
\end{table}

\paragraph{MNIST-Binary.}
Table~\ref{tab:gap_dynamics_mnist_binary} shows gap dynamics on MNIST-Binary. With binary inputs, computation gap is zero, so this table isolates pure selection gap.

Hard-ST and Hard-ST+CAGE achieve exactly 0\% gap at all temperatures and layers. The cells show ``$+0.0$'' with no range because the gap is identically zero throughout training. This is the strongest validation of forward alignment: when training uses the same argmax selection as inference, there is no selection mismatch by construction.

\begin{table}[h]
\centering
\caption{Gap trajectory on MNIST-Binary. Format: Final [min, max] (\%). With binary inputs, computation gap = 0, so this isolates pure selection gap. Hard-ST achieves exactly 0\% at all temperatures.}
\label{tab:gap_dynamics_mnist_binary}
\scriptsize
\begin{tabular}{llccccc}
\toprule
\textbf{Method} & \textbf{L} & $\tau$=0.05 & $\tau$=0.1 & $\tau$=0.5 & $\tau$=1.0 & $\tau$=2.0 \\
\midrule
Soft-Mix & 4 & $+0.0$ [-0.1,+0.0] & $+0.0$ & $+0.0$ [-0.1,+0.1] & $+0.0$ [-0.1,+0.2] & $+0.0$ [-0.1,+0.3] \\
 & 5 & $+0.0$ & $+0.0$ & $+0.0$ [-0.1,+0.1] & $+0.0$ [-0.1,+0.1] & $+0.0$ [-0.1,+1.6] \\
 & 6 & $+0.0$ & $+0.0$ & $+0.0$ [-0.1,+0.1] & $+0.0$ [-0.1,+0.1] & $+0.6$ [+0,+24] \\
\midrule
Soft-Gumbel & 4 & $+0.3$ [-31,+1] & $+0.1$ [-14,+1] & $+0.1$ [-0.3,+1.8] & $+0.2$ [+0,+9] & $+0.2$ [+0,+23] \\
 & 5 & $-0.1$ [-35,+1] & $-0.5$ [-31,+2] & $+0.2$ [-0.4,+4.3] & $+0.2$ [+0,+18] & $+0.7$ [+0,+35] \\
 & 6 & $+0.7$ [-30,+2] & $+0.2$ [-25,+2] & $+0.2$ [-1,+7] & $+0.2$ [+0,+17] & $+2.1$ [+2,+40] \\
\midrule
Hard-ST & 4 & $+0.0$ & $+0.0$ & $+0.0$ & $+0.0$ & $+0.0$ \\
 & 5 & $+0.0$ & $+0.0$ & $+0.0$ & $+0.0$ & $+0.0$ \\
 & 6 & $+0.0$ & $+0.0$ & $+0.0$ & $+0.0$ & $+0.0$ \\
\midrule
Gumbel-ST & 4 & $-45.1$ [-47,-41] & $-44.7$ [-55,-43] & $+0.0$ [-0.5,+0.1] & $+0.0$ [-0.2,+0.3] & $+0.0$ [-0.1,+0.5] \\
 & 5 & $-38.7$ [-39,-28] & $-57.9$ [-59.9,-55.3] & $+0.0$ [-0.7,+0.1] & $+0.0$ [-0.2,+0.3] & $+0.0$ [-0.2,+0.7] \\
 & 6 & $-40.4$ [-41,-29] & $-58.8$ [-61,-55] & $-0.1$ [-1.3,+0.1] & $+0.0$ [-0.3,+0.3] & $+0.0$ [-0.3,+1.1] \\
\midrule
Hard-ST+CAGE & 4 & $+0.0$ & $+0.0$ & $+0.0$ & $+0.0$ & $+0.0$ \\
 & 5 & $+0.0$ & $+0.0$ & $+0.0$ & $+0.0$ & $+0.0$ \\
 & 6 & $+0.0$ & $+0.0$ & $+0.0$ & $+0.0$ & $+0.0$ \\
\midrule
Gumbel-ST+CAGE & 4 & $+0.1$ [-0.2,+0.6] & $+0.0$ [-0.2,+0.8] & $-0.1$ [-0.3,+0.7] & $+0.0$ [-0.1,+0.4] & $+0.0$ [-0.2,+0.4] \\
 & 5 & $+0.0$ [-0.2,+1.7] & $+0.0$ [-0.2,+2.3] & $+0.0$ [-0.2,+2.0] & $+0.0$ [-0.2,+0.8] & $+0.0$ [-0.2,+0.7] \\
 & 6 & $+0.0$ [-0.3,+4.5] & $+0.1$ [-0,+9] & $+0.0$ [-0.2,+3.5] & $+0.0$ [-0.2,+1.6] & $+0.0$ [-0.2,+2.2] \\
\bottomrule
\end{tabular}
\end{table}

\paragraph{CIFAR-10-Binary.}

Table~\ref{tab:gap_dynamics_cifar10} shows gap dynamics on CIFAR-10-Binary binarized with 31 evenly-spaced thresholds per channel.

CIFAR-10-Binary shows smaller transient gaps for Soft-Gumbel compared to MNIST (max $\sim$3.5\% vs $\sim$40\%). This may be because the harder classification task requires more commitment to specific gates, reducing the benefit of hedging. Gumbel-ST still fails at low $\tau$ but less catastrophically ($-$14\% vs $-$59\%).

\begin{table}[t]
\centering
\caption{Gap trajectory on CIFAR-10-Binary. Format: Final [min, max] (\%). Hard-ST achieves exactly 0\% gap; Gumbel-ST fails at low $\tau$ but less severely than on MNIST.}
\label{tab:gap_dynamics_cifar10}
\scriptsize
\begin{tabular}{llccccc}
\toprule
\textbf{Method} & \textbf{L} & $\tau$=0.05 & $\tau$=0.1 & $\tau$=0.5 & $\tau$=1.0 & $\tau$=2.0 \\
\midrule
Soft-Mix & 4 & $+0.0$ [-0.1,+0.1] & $+0.0$ [-0.1,+0.1] & $+0.0$ [-0.1,+0.1] & $+0.0$ [-0.3,+0.3] & $+0.1$ [-0.3,+0.5] \\
 & 5 & $+0.0$ [-0.1,+0.1] & $+0.0$ [-0.1,+0.1] & $+0.0$ [-0.1,+0.1] & $+0.0$ [-0.3,+0.3] & $-0.1$ [-0.2,+1.5] \\
 & 6 & $+0.0$ [-0.1,+0.1] & $+0.0$ [-0.1,+0.1] & $+0.0$ [-0.1,+0.1] & $+0.0$ [-0.2,+0.3] & $+0.3$ [-0,+9] \\
\midrule
Soft-Gumbel & 4 & $-0.2$ [-2.3,+0.4] & $+0.0$ [-2.0,+0.4] & $+0.0$ [-0.7,+0.5] & $+0.0$ [-0.3,+0.3] & $+0.1$ [-0.2,+0.4] \\
 & 5 & $+0.0$ [-2.5,+0.3] & $-0.2$ [-2.6,+0.4] & $+0.0$ [-0.8,+0.7] & $+0.0$ [-0.3,+0.5] & $+0.1$ [-0.3,+0.4] \\
 & 6 & $-0.3$ [-3.5,+0.2] & $+0.2$ [-2.3,+0.4] & $+0.0$ [-0.7,+0.5] & $+0.0$ [-0.3,+0.5] & $+0.0$ [-0.3,+0.4] \\
\midrule
Hard-ST & 4 & $+0.0$ & $+0.0$ & $+0.0$ & $+0.0$ & $+0.0$ \\
 & 5 & $+0.0$ & $+0.0$ & $+0.0$ & $+0.0$ & $+0.0$ \\
 & 6 & $+0.0$ & $+0.0$ & $+0.0$ & $+0.0$ & $+0.0$ \\
\midrule
Gumbel-ST & 4 & $-9.6$ [-11,-4] & $-9.5$ [-11.8,-8.2] & $-0.2$ [-0.8,+0.5] & $+0.0$ [-0.3,+0.2] & $+0.0$ [-0.4,+0.5] \\
 & 5 & $-5.2$ [-8,-2] & $-11.4$ [-14.3,-10.4] & $+0.0$ [-0.9,+0.2] & $+0.0$ [-0.4,+0.3] & $-0.1$ [-0.4,+0.3] \\
 & 6 & $-7.5$ [-8,-0] & $-13.7$ [-16.4,-12.6] & $-0.4$ [-1.2,+0.2] & $+0.0$ [-0.4,+0.3] & $-0.1$ [-0.3,+0.4] \\
\midrule
Hard-ST+CAGE & 4 & $+0.0$ & $+0.0$ & $+0.0$ & $+0.0$ & $+0.0$ \\
 & 5 & $+0.0$ & $+0.0$ & $+0.0$ & $+0.0$ & $+0.0$ \\
 & 6 & $+0.0$ & $+0.0$ & $+0.0$ & $+0.0$ & $+0.0$ \\
\midrule
Gumbel-ST+CAGE & 4 & $-0.1$ [-0.6,+0.3] & $-0.1$ [-0.6,+0.3] & $-0.1$ [-0.6,+0.3] & $-0.1$ [-0.4,+0.4] & $+0.0$ [-0.6,+0.3] \\
 & 5 & $+0.0$ [-0.6,+0.3] & $-0.2$ [-0.6,+0.3] & $-0.1$ [-0.6,+0.3] & $+0.0$ [-0.8,+0.3] & $-0.2$ [-0.5,+0.3] \\
 & 6 & $+0.0$ [-0.6,+0.4] & $+0.0$ [-0.7,+0.4] & $-0.1$ [-0.7,+0.3] & $+0.0$ [-0.5,+0.3] & $-0.1$ [-0.5,+0.4] \\
\bottomrule
\end{tabular}
\end{table}

\subsection{Heatmap Visualization}
\label{app:heatmap}

Figure~\ref{fig:heatmap_app} visualizes gap across the full (layer $\times$ temperature) grid, providing an overview of all experimental conditions.

\begin{figure*}[t]
    \centering
    \includegraphics[width=0.85\textwidth]{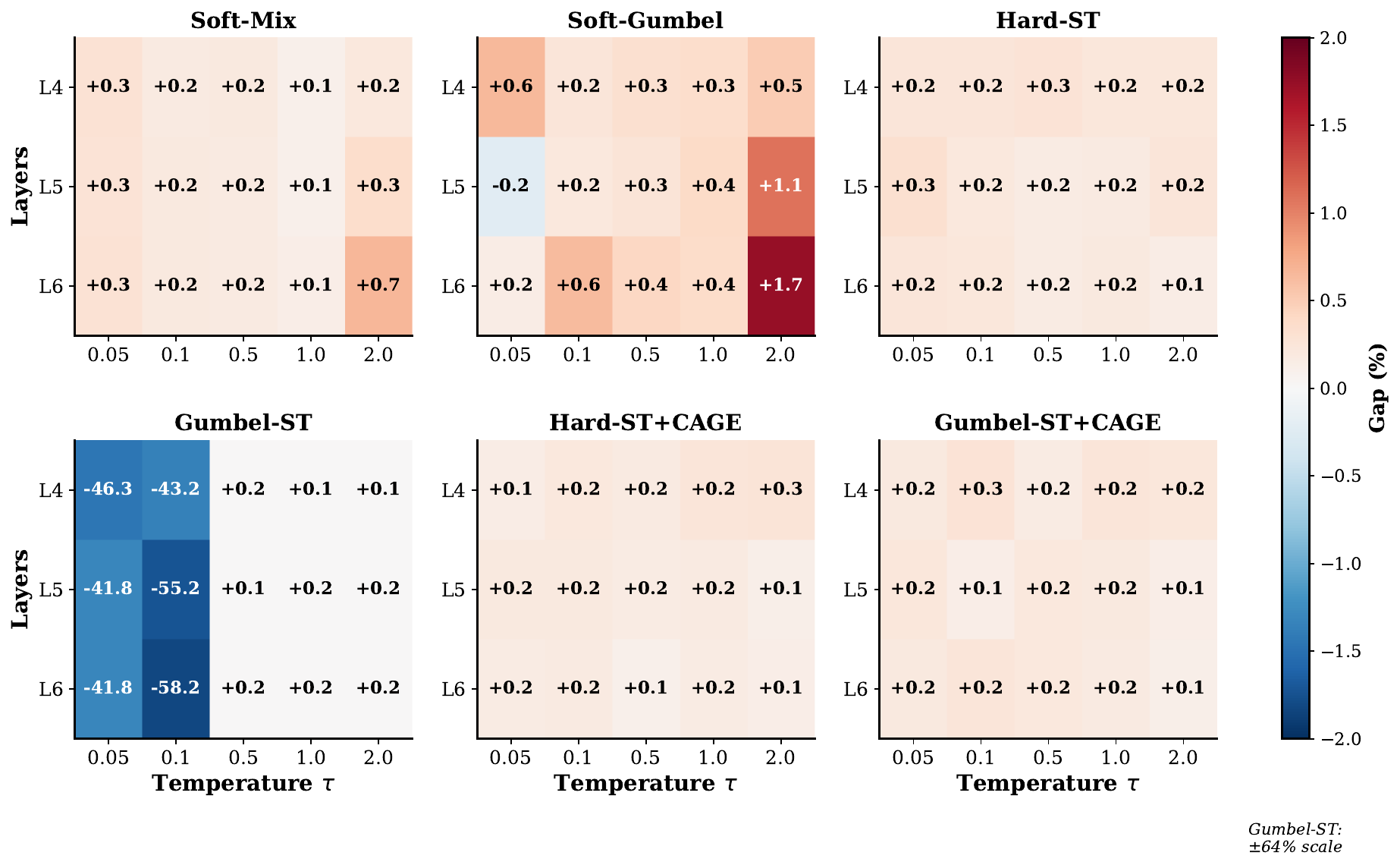}
    \caption{Gap heatmaps on MNIST. Hard-ST shows uniform near-zero gap (light colors) across all conditions. Gumbel-ST exhibits catastrophic negative gap at low $\tau$ (dark blue = training failure where deployment outperforms training). CAGE eliminates this failure, restoring uniform small gap.}
    \label{fig:heatmap_app}
\end{figure*}


\subsection{Statistical Tests}
\label{app:statistical_tests}

We report statistical significance tests for the key claims in this paper. All tests use paired $t$-tests with $n=3$ seeds per configuration. With small sample sizes and low variance, even small differences can achieve statistical significance. We therefore distinguish between \textbf{statistical significance} ($p < 0.05$) and \textbf{practical significance} ($|\Delta\text{Accuracy}| > 5\%$ or $|\Delta\text{Gap}| > 5\%$).

Table~\ref{tab:statistical_significance} reports results for MNIST-Binary with $L$=6 layers, which isolates selection gap (computation gap is zero for binary inputs).

\paragraph{Findings.}

\textbf{Gumbel-ST failure is both statistically and practically significant.}
The accuracy drop from $\tau=2.0$ to $\tau=0.05$ is highly significant ($t=-36.16$, $p<0.001$, $d=31.21$) and represents catastrophic training failure (47 percentage points accuracy loss).

\textbf{Hard-ST shows statistical but not practical sensitivity to $\tau$.}
While the accuracy difference is statistically significant ($p=0.005$) due to low variance across seeds, the 1.1 percentage point range is negligible compared to Gumbel-ST's 47 percentage point range. The claim of ``temperature robustness'' refers to practical robustness.

\textbf{CAGE effectively resolves Gumbel-ST's failure.}
At $\tau=0.05$, CAGE improves Gumbel-ST accuracy by 47 percentage points ($t=-38.81$, $p<0.001$, $d=31.28$), restoring performance to the level achieved at higher temperatures.

\textbf{Hard-ST achieves zero selection gap.}
The gap difference between Hard-ST (0\%) and Soft-Gumbel (40\% peak) at $\tau=2.0$ is statistically significant ($p=0.024$) with large effect size ($d=5.15$).

\textbf{Forward structure effect shows practical significance.}
The comparison between Soft-Gumbel and Gumbel-ST gaps (same Gumbel noise, different forward pass) shows a clear practical difference (40\% vs 1\% gap) with marginal statistical significance($p=0.093$, $d=2.61$). The elevated $p$-value reflects high variance in Soft-Gumbel's gap across seeds; the direction consistently supports the forward alignment hypothesis.

\begin{table}[h]
\centering
\caption{Statistical significance of key claims (MNIST-Binary, $L$=6, $n$=3 seeds). 
Effect size $d$ is Cohen's $d$~\citep{cohen1988statistical}.}
\label{tab:statistical_significance}
\small
\begin{tabular}{lcccccc}
\toprule
Claim & Comparison & $t$ & $p$ & $d$ & Stat. & Prac. \\
\midrule
\multicolumn{7}{l}{\textit{Temperature sensitivity (accuracy)}} \\
\quad Gumbel-ST fails at low $\tau$ & $\tau$=0.05 vs $\tau$=2.0 & $-$36.16 & $<$.001 & 31.21 & *** & \checkmark \\
\quad Hard-ST is $\tau$-robust & $\tau$=0.05 vs $\tau$=2.0 & $-$14.53 & .005 & 16.26 & ** & -- \\
\addlinespace
\multicolumn{7}{l}{\textit{CAGE effectiveness}} \\
\quad CAGE fixes Gumbel-ST & Gumbel-ST vs +CAGE ($\tau$=0.05) & $-$38.81 & $<$.001 & 31.28 & *** & \checkmark \\
\addlinespace
\multicolumn{7}{l}{\textit{Selection gap comparisons (at $\tau$=2.0)}} \\
\quad Hard-ST has zero selection gap & Hard-ST vs Soft-Gumbel & $-$6.30 & .024 & 5.15 & * & \checkmark \\
\addlinespace
\multicolumn{7}{l}{\textit{Forward vs backward structure}} \\
\quad Forward type determines gap & Soft-Gumbel vs Gumbel-ST gap & 3.05 & .093 & 2.61 & ns & \checkmark \\
\bottomrule
\end{tabular}
\\[1mm]
\footnotesize
\textit{Significance:} $^{***}p<0.001$, $^{**}p<0.01$, $^{*}p<0.05$, ns = not significant. 
\textit{Practical:} \checkmark = $|\Delta| > 5\%$ accuracy difference.
\end{table}

\paragraph{Cross-dataset consistency.}
Table~\ref{tab:stat_cross_dataset} confirms these findings hold across datasets.

\begin{table}[htb]
\centering
\caption{Cross-dataset verification of key claims ($L$=6, significance level).}
\label{tab:stat_cross_dataset}
\small
\begin{tabular}{lccc}
\toprule
Claim & MNIST & MNIST-Binary & CIFAR-10-Binary \\
\midrule
Gumbel-ST fails at low $\tau$ & *** & *** & *** \\
Hard-ST is $\tau$-robust & **$^\dagger$ & **$^\dagger$ & ***$^\dagger$ \\
CAGE fixes Gumbel-ST & *** & *** & *** \\
\bottomrule
\end{tabular}
\\[1mm]
\footnotesize
\textit{Significance:} $^{***}p<0.001$, $^{**}p<0.01$, $^{*}p<0.05$.\\
$^\dagger$Statistically significant but practically robust: MNIST datasets show $<$2\% accuracy range, CIFAR-10-Binary shows $\sim$6\% range (all far below Gumbel-ST's $>$39\%).
\end{table}

\subsection{Gate Usage Analysis}
\label{app:gate_usage}

Figure~\ref{fig:gate_usage_app} shows the distribution of selected gates across all nodes in the trained networks.

\begin{figure}[htb]
    \centering
    \includegraphics[width=0.65\linewidth]{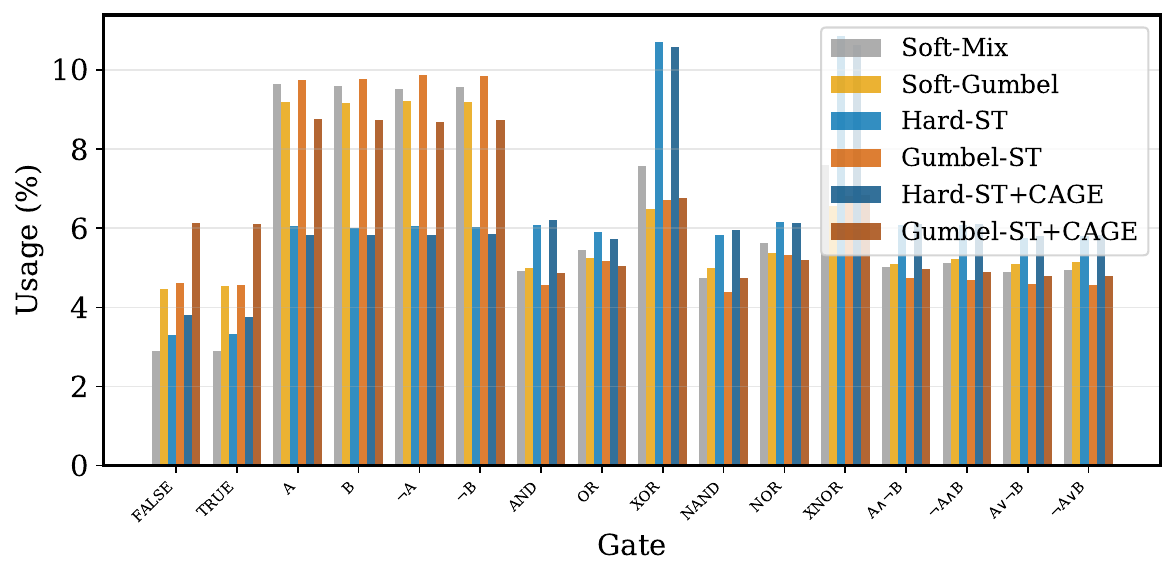}
\caption{Gate usage distribution (MNIST-Binary, $L$=6, $\tau$=1.0). Hard-ST methods show pronounced XOR/XNOR preference ($\sim$11\% each vs.\ $\sim$6\% expected under uniform selection), while soft methods distribute usage more uniformly. This suggests hard selection encourages commitment to more expressive Boolean operations.}
    \label{fig:gate_usage_app}
\end{figure}

%% file: appendix_input_distribution.tex
\section{Input Distribution Analysis}
\label{app:input_distribution}

This appendix analyzes how input distribution affects the computation gap between soft and hard logic gates, explaining why computation gap is small in our experiments.

\subsection{Gate Definitions}
\label{app:gate_definitions}

We use product t-norm and probabilistic t-conorm relaxations~\citep{klement2000triangular, vankrieken2022analyzing, petersen2022deep} for soft gates, with standard negation ($\neg a \to 1-a$). Hard gates use Boolean operations with threshold $\theta = 0.5$. Table~\ref{tab:all_gates_formulas} lists all 16 gates.

\begin{table}[h]
\centering
\caption{All 16 two-input Boolean gates with soft (differentiable) and hard (Boolean) formulas.}
\label{tab:all_gates_formulas}
\small
\begin{tabular}{cllll}
\toprule
\textbf{Index} & \textbf{Gate} & \textbf{Boolean} & \textbf{Soft Formula} & \textbf{Hard Formula} \\
\midrule
0 & FALSE & $0$ & $0$ & $0$ \\
1 & AND & $a \land b$ & $ab$ & $\mathds{1}[a > 0.5] \land \mathds{1}[b > 0.5]$ \\
2 & A$\land\neg$B & $a \land \neg b$ & $a(1-b)$ & $\mathds{1}[a > 0.5] \land \mathds{1}[b \leq 0.5]$ \\
3 & A & $a$ & $a$ & $\mathds{1}[a > 0.5]$ \\
4 & $\neg$A$\land$B & $\neg a \land b$ & $(1-a)b$ & $\mathds{1}[a \leq 0.5] \land \mathds{1}[b > 0.5]$ \\
5 & B & $b$ & $b$ & $\mathds{1}[b > 0.5]$ \\
6 & XOR & $a \oplus b$ & $a+b-2ab$ & $\mathds{1}[a > 0.5] \oplus \mathds{1}[b > 0.5]$ \\
7 & OR & $a \lor b$ & $a+b-ab$ & $\mathds{1}[a > 0.5] \lor \mathds{1}[b > 0.5]$ \\
8 & NOR & $\neg(a \lor b)$ & $1-a-b+ab$ & $\neg(\mathds{1}[a > 0.5] \lor \mathds{1}[b > 0.5])$ \\
9 & XNOR & $\neg(a \oplus b)$ & $1-a-b+2ab$ & $\neg(\mathds{1}[a > 0.5] \oplus \mathds{1}[b > 0.5])$ \\
10 & $\neg$B & $\neg b$ & $1-b$ & $\mathds{1}[b \leq 0.5]$ \\
11 & A$\lor\neg$B & $a \lor \neg b$ & $1-b+ab$ & $\mathds{1}[a > 0.5] \lor \mathds{1}[b \leq 0.5]$ \\
12 & $\neg$A & $\neg a$ & $1-a$ & $\mathds{1}[a \leq 0.5]$ \\
13 & $\neg$A$\lor$B & $\neg a \lor b$ & $1-a+ab$ & $\mathds{1}[a \leq 0.5] \lor \mathds{1}[b > 0.5]$ \\
14 & NAND & $\neg(a \land b)$ & $1-ab$ & $\neg(\mathds{1}[a > 0.5] \land \mathds{1}[b > 0.5])$ \\
15 & TRUE & $1$ & $1$ & $1$ \\
\bottomrule
\end{tabular}
\end{table}

\subsection{Computation Gap Bounds}

The computation gap for a gate $g$ is the expected absolute difference between soft and hard outputs:
\begin{equation}
    \Delta_g = \mathbb{E}_{a, b \sim \mathcal{D}} \left[ \left| g^{\text{soft}}(a, b) - g^{\text{hard}}(a, b) \right| \right].
    \label{eq:comp_gap_def}
\end{equation}

The gap depends critically on input distribution. Binary inputs $\{0, 1\}$ achieve zero gap since soft and hard gates produce identical outputs. Uniform $[0,1]$ inputs yield moderate gap ($\sim$20\%). The upper bound occurs at the decision boundary ($a \approx 0.5$). Table~\ref{tab:gap_bounds} summarizes these bounds.

\begin{table}[h]
\centering
\caption{Computation gap for representative gates under different input distributions.}
\label{tab:gap_bounds}
\small
\begin{tabular}{lcccc}
\toprule
\textbf{Distribution} & \textbf{Unary} & \textbf{AND/OR} & \textbf{XOR} \\
\midrule
Binary $\{0,1\}$ (lower bound) & 0\% & 0\% & 0\% \\
MNIST (91\% binary-like) & $\sim$2\% & $\sim$3\% & $\sim$5\% \\
Uniform $[0,1]$ & 25\% & 21.9\% & 16.7\% \\
At boundary $a{=}0.5$ (upper) & 50\% & 75\% & 50\% \\
\bottomrule
\end{tabular}
\end{table}

For Uniform $[0,1]$ inputs, the gaps can be computed analytically. Unary gates have $\Delta = 0.25$. AND and OR gates have $\Delta = 0.21875$ (by quadrant integration; contributions swap between gates but totals match). XOR gates have $\Delta = 0.167$. Negation preserves gap: $\Delta_{\text{NAND}} = \Delta_{\text{AND}}$, etc. Table~\ref{tab:app_all_gates_gap} groups all 16 gates by gap value.

\begin{table}[h]
\centering
\caption{Expected computation gap for all 16 gates on Uniform [0,1] inputs.}
\label{tab:app_all_gates_gap}
\small
\begin{tabular}{llc}
\toprule
\textbf{Category} & \textbf{Gates} & \textbf{Gap (\%)} \\
\midrule
Constants & FALSE, TRUE & 0.0 \\
XOR-type & XOR, XNOR & 16.7 \\
AND/OR-type & AND, OR, NAND, NOR, A$\land\neg$B, $\neg$A$\land$B, A$\lor\neg$B, $\neg$A$\lor$B & 21.9 \\
Unary & A, B, $\neg$A, $\neg$B & 25.0 \\
\bottomrule
\end{tabular}
\end{table}

\subsection{MNIST Pixel Distribution}

MNIST is highly sparse: 80.9\% of pixels are exactly zero (background), and only 8.7\% fall in the intermediate range $[0.1, 0.9]$. Overall, 91.3\% of pixels are binary-like ($<0.1$ or $>0.9$). This explains why computation gap is small in our experiments.

\begin{table}[h]
\centering
\caption{MNIST pixel value distribution.}
\label{tab:app_mnist_dist}
\small
\begin{tabular}{lcc}
\toprule
\textbf{Pixel Range} & \textbf{MNIST [0,1]} & \textbf{MNIST-Binary} \\
\midrule
Exactly 0 & 80.9\% & 86.7\% \\
$(0, 0.1)$ & 1.7\% & 0.0\% \\
$[0.1, 0.9]$ & 8.7\% & 0.0\% \\
$(0.9, 1.0]$ & 8.8\% & 13.3\% \\
\midrule
\textbf{Binary-like} ($<0.1$ or $>0.9$) & \textbf{91.3\%} & \textbf{100.0\%} \\
\bottomrule
\end{tabular}
\end{table}

\subsection{Implications}

With 91.3\% binary-like pixels, MNIST is close to the best-case scenario where binary inputs yield 0\% gap. MNIST-Binary achieves exactly zero computation gap, isolating pure selection gap. This explains why our experiments show selection gap as the dominant component: MNIST's distribution minimizes computation gap, making selection gap the primary differentiator between methods.

For applications with truly continuous inputs (e.g., Uniform $[0,1]$), computation gap becomes significant ($\sim$20\%), and inputs concentrated near the decision boundary can reach 50--75\% gap.